\documentclass{article}
\usepackage{custom_tex}
\usepackage{dsfont}
\usepackage{enumitem}
\usepackage{multirow}
\usepackage{booktabs}
\usepackage{algorithmic}
\usepackage{algorithm}
\title{Collaborative Learning of Discrete Distributions under Heterogeneity and Communication Constraints}

\author{
Xinmeng Huang,\footnote{Equal Contribution.} \footnote{Graduate Group in Applied Mathematics and Computational Science, Univ. of Pennsylvania. \texttt{xinmengh@sas.upenn.edu}.}
\,
Donghwan Lee,\footnotemark[1]\,\,\footnote{Graduate Group in Applied Mathematics and Computational Science, Univ. of Pennsylvania. \texttt{dh7401@sas.upenn.edu}.}
\,
Edgar Dobriban,\footnote{Department of Statistics and Data Science, Univ. of Pennsylvania. \texttt{dobriban@wharton.upenn.edu}.}
\,
and Hamed Hassani\footnote{Department of Electrical and Systems Engineering, Univ. of Pennsylvania. \texttt{hassani@seas.upenn.edu}.}
}
\date{\today}

\def\cB{\mathcal{B}}
\def\cC{\mathcal{C}}
\def\cE{\mathcal{E}}
\def\cF{\mathcal{F}}
\def\cI{\mathcal{I}}
\def\cK{\mathcal{K}}
\def\cP{\mathcal{P}}
\def\cQ{\mathcal{Q}}
\def\cU{\mathcal{U}}

\def\E{\mathbb{E}}

\def\PP{\mathbb{P}}
\def\RR{\mathbb{R}}

\def\bb{\mathbf{b}}
\def\bv{\mathbf{v}}
\def\bp{\mathbf{p}}
\def\bx{\mathbf{x}}
\def\by{\mathbf{y}}

\def\Bin{\mathrm{Binom}}

\def\Cat{\mathrm{Cat}}
\def\proj{\mathrm{Proj}}
\def\supp{\mathrm{supp}}
\def\trmean{\mathrm{trmean}}
\def\med{\mathrm{median}}

\newcommand{\eg}{\emph{e.g.}}
\newcommand{\ie}{\emph{i.e.}}

\DeclareFontFamily{U}{mathx}{}
\DeclareFontShape{U}{mathx}{m}{n}{ <-> mathx10 }{}
\DeclareSymbolFont{mathx}{U}{mathx}{m}{n}
\DeclareFontSubstitution{U}{mathx}{m}{n}
\DeclareMathAccent{\widecheck}{0}{mathx}{"71}

\usepackage[margin=1.1in]{geometry}

\begin{document}
\maketitle

\begin{abstract}
In modern machine learning, users often have to collaborate to learn the distribution of the data. 
Communication can be a significant bottleneck. Prior work has studied homogeneous users---\ie, whose data follow the same discrete distribution---and has provided optimal communication-efficient methods for estimating that distribution.
However, these methods rely heavily on homogeneity, and are less applicable in the common case when users' discrete distributions are heterogeneous. 
Here we consider a natural and tractable model of heterogeneity, where users' discrete distributions only vary sparsely, on a small number of entries. 
We propose a novel two-stage method named SHIFT: 
First, the users collaborate by communicating with the server to learn a central distribution; relying on methods from robust statistics. 
Then, the learned central distribution is fine-tuned to estimate their respective individual distribution. 
We show that SHIFT is minimax optimal in our model of heterogeneity and under communication constraints. 
Further, we provide experimental results using both synthetic data and $n$-gram frequency estimation in the text domain, which corroborate its efficiency.
\end{abstract}


\section{Introduction}

Research on learning from data distributed over multiple computational units (machines, users, devices) has grown in recent years, as data is commonly generated by multiple users, such as  smart devices  and wireless sensors. 
While many works focus on learning predictive models with distributed data, learning the data distribution itself is also increasingly popular \cite{Wang2022FederatedAO,pmlr-v132-acharya21b,chen2021breaking,Chen2021PointwiseBF}. 
In various applications, it is often required to reconstruct the data distribution from scattered measurements. Examples include sensor networks and P2P (Peer2Peer) systems, load balancing and query processing (see, \eg, \cite{Nowak2003DistributedEA,Zhou2012EffectiveDD,Slavov2013AGA} and references therein).
In these scenarios, communication costs and bandwidth are often bottlenecks on the performance of learning algorithms \cite{Garg2014OnCC,pmlr-v23-balcan12a,huang2022lb,Daum2012EfficientPF}.
The bottlenecks become even more severe in federated analytics \cite{Kairouz2021AdvancesAO}, 
where many users coordinate with a server to learn central models, while communication via the wireless links is typically expensive and operates at low rates.

This paper considers learning high-dimensional discrete distributions from user data in the distributed setting.
In this setting, several communication-efficient methods have been proposed, and their optimality  under communication constraints has been established under various models \cite{Han2018DistributedSE,Barnes2019LowerBF,Han2021GeometricLB,Acharya2020InferenceUI,pmlr-v132-acharya21b,chen2021breaking,Chen2021PointwiseBF}. 
However, the key challenge of  \emph{heterogeneity}, \ie, that users' distributions can differ, is rarely considered. 
Heterogeneity is common, as users inevitably have 
unique characteristics \cite{NEURIPS2020_222afbe0}. 
Meanwhile, heterogeneity can cause a significant performance drop for learning algorithms designed only for i.i.d data \cite{McMahan2017CommunicationEfficientLO,Li2020OnTC,Dobriban2018DistributedLR}. 
To use all the data, one needs to learn some central structure, transferable to all individual users. 
Then one may locally learn--or, finetune---some unique components for each user \cite{Du2021FewShotLV,Tripuraneni2021ProvableMO,Collins2021ExploitingSR,Xu2021LearningAB}. 

To study this paradigm, we first need to introduce a suitable model of heterogeneity.
We consider, as an example, the heterogeneous frequencies of words across different texts, \eg, news articles, books, plays (tragedies and comedies), viewed as users. 
Most words appear with nearly the same probabilities in different texts, 
however, a few can have very different probabilities, such as  ``sorrowful'' being common in tragedies and ``convivial'' being common in comedies. 
Motivated by this, we formulate a model of sparse  heterogeneity. 
Specifically, suppose that the discrete distributions of all users differ from an underlying  central distribution in at most $s \geq0$ entries, where $s$ is much smaller than the dimension $d \geq0$. 
Sparse heterogeneity is relevant to applications such as recommendation systems \cite{Hu2019HERSMI,Qian2015StructuredSR,Liu2019RecommenderSW,Bastani2021PredictingWP} and medical risk scoring \cite{Subbaswamy2019FromDT,QuioneroCandela2009DatasetSI,Mullainathan2017DoesML}.

However, given data generated by multiple distributions with sparse heterogeneity, previous works \cite{Han2018DistributedSE,Acharya2020InferenceUI,pmlr-v132-acharya21b,Chen2021PointwiseBF} either do not use all the data, 
or suffer from bias due to heterogeneity that does not vanish as the sample size increases.
Here we propose a novel \underline{\textbf{s}}parse \underline{\textbf{h}}eterogeneity-\underline{\textbf{i}}nspired collaboration and \underline{\textbf{f}}ine-\underline{\textbf{t}}uning method (SHIFT) where we first collaboratively learn the central distribution, 
and then fine-tune the central estimate to individual distributions. 
Our method makes full use of heterogeneous data, leading to a significant improvement in error rates compared to prior methods. See Table \ref{tab:comparsion} for an overview, explained in detail later.

\begin{table}[t]
\centering 
\caption{\small Estimation error $\E[\|\widehat{\bp}^t-\bp^t\|_2^2]$ of various methods when $n$ is sufficiently large:
$\bp^t$ is the test distribution,
$\widehat{\bp}^t$ is the estimator,
$\boldsymbol{\delta}^t= T^{-1}\sum_{t^\prime \in[T]}\bp^{t^\prime}-\bp^t$ is a non-vanishing measure of heterogeneity. 
See Section \ref{cont} for other notations.
Constants and logarithmic factors are omitted for clarity. 
The  ``data usage'' column indicates whether the estimate is obtained for each cluster separately or by pooling data. }
\label{tab:comparsion}
\begin{tabular}{lccc}
\toprule
Method & Estimation Error & Data Usage & Bound Type\\
\midrule
Unif. Group./Hash. \cite{{Han2018DistributedSE}} & $O\left(\frac{d}{2^b n}\right)$ & Separate & Upper \\
Unif. Group./Hash. \cite{{Han2018DistributedSE}}  & $O\left(\|\boldsymbol{\delta}^t\|_2^2+\frac{d}{2^bTn}\right)$ & Pool & Upper \\
Localize-then-Refine$^*$ \cite{Chen2021PointwiseBF} & ${O}\left(\frac{\|\bp^t\|_{1/2}}{2^b n}
\right)$ & Separate & Upper \\
Localize-then-Refine$^*$ \cite{Chen2021PointwiseBF} & ${O}\left(\|\boldsymbol{\delta}^t\|_2^2+\frac{\|\frac{1}{T}\sum_{t^\prime \in[T]}\bp^{t^\prime}\|_{1/2}}{2^bT n}\right)$ & Pool & Upper\\
\midrule
SHIFT (Theorem \ref{thm:median-collab-final}) & $\tilde{O}\left(\frac{\max\{2^b,s\}}{ 2^b n}+\frac{d}{2^b Tn}\right)$ & $-$ & Upper \\
SHIFT (Theorem \ref{thm:collab-lower-bound}) & ${\Omega}\left(\frac{\max\{2^b,s\}}{2^b n}+\frac{d}{2^b Tn}\right)$ & $-$ & Lower \\
\bottomrule
	\multicolumn{4}{l}{$^*$\,\footnotesize{This method \cite{Chen2021PointwiseBF} requires interactive communication protocols, while other methods are non-interactive.}}
\end{tabular}
\end{table}

\subsection{Contributions}\label{cont}
We consider the problem of learning $d$-dimensional distributions with $s$-sparse heterogeneity. 
We assume there are $T$ clusters of user datapoints, 
and allow each datapoint to be transmitted in a message with at most $b$ bits of information to the server. 
Our setting embraces heterogeneous data and thus is a significant generalization of the models from \cite{Han2018DistributedSE,Barnes2019LowerBF,Han2021GeometricLB,Acharya2020InferenceUI}.
Our technical contributions are as follows:
\begin{itemize}
    \item We propose the SHIFT method to  learn heterogeneous distributions with collaboration and tuning, in a sample-efficient manner. 
    Our method can, in principle, be used with an arbitrary robust estimate of the probability of each entry/coordinate. 
    When entry-wise median and trimmed mean are used, we provide upper bounds on the estimation error of individual distributions in the $\ell_2$ and $\ell_1$ norms.
    We show a factor of (the order of) $\min\{T,d/\max\{s,2^b\}\}$ improvement in sample complexity compared to previous works; showing the benefit of collaboration (large $T$) and sparsity (small $s$), despite communication constraints and heterogeneity.

    \item To justify the optimality of our method, we prove minimax lower bounds on the estimation errors of individual distributions in the $\ell_2$ and $\ell_1$ norms, holding for all, possibly interactive, learning methods. 
    These lower bounds, combined with our upper bounds, imply that our median-based method is minimax optimal.
    \item We support our method with experiments on both synthetic and empirical datasets, showing a significant improvement over previous methods.
\end{itemize}

\subsection{Related Works}\label{sec:related-work} 

\paragraph{Learning with heterogeneity.}
Learning with heterogeneity is commonly found in the broader context of multi-task learning \cite{RCaruana,Thrun1998LearningTL,Baxter2000AMO} and federated learning \cite{Achituve2021PersonalizedFL,Shamsian2021PersonalizedFL,Grimberg2021OptimalMA,smith2017federated}, where a central model or representation is learned from multiple heterogeneous datasets. 
These central representations can be useful for few-shot learning, \ie, for new problems with a small sample size \cite{Sun2017RevisitingUE,Goyal2019ScalingAB} due to their ability to adapt to new tasks efficiently. 
In heterogeneous linear regression, \cite{Tripuraneni2021ProvableMO,Du2021FewShotLV} show improved sample complexities by assuming a low dimensional central representation, compared to the i.i.d. setting \cite{Grimberg2021OptimalMA,Maurer2016TheBO}. 
Related results are proved in \cite{Collins2021ExploitingSR} for personalized federated learning. 
\cite{Xu2021LearningAB} study a bandit problem where the unknown parameter in each dataset equals a global parameter plus a sparse instance-specific term. We study a different setting, learning distributions with sparse heterogeneity under communication constraints.

\paragraph{Estimating distributions under communication constraints. }
Estimating discrete distributions has a rich literature \cite{lehmann2006theory,agresti2003categorical}. Under communication constraints, \cite{Han2018DistributedSE,Barnes2019LowerBF,Han2021GeometricLB,Acharya2020InferenceUI} consider the non-interactive scenario, and establish the minimax optimal rates, in terms of data dimension and communication budget, via potentially shared randomness, when all users' data is homogeneous. 
The optimality for the general
interactive (or, blackboard) methods is developed by \cite{Acharya2020GeneralLB}. 
A few  works study the estimation of  sparse distributions.
In particular, \cite{pmlr-v132-acharya21b} consider $s$-sparse distributions and establish minimax optimal rates under communication and privacy constraints, which are further improved  by  localization strategies in \cite{chen2021breaking}. 
Complementary to minimax rates, \cite{Chen2021PointwiseBF} provides pointwise rates, governed by the half-norm of the distribution to be learned, instead of its dimension. Our setting embraces heterogeneous data, and thus is a generalization of the one studied in above works.

\paragraph{Robust estimation \& learning.} Robust statistics and learning study algorithms resilient to unknown data corruption \cite{huber1981robust,hampel2011robust, Anscombe1960RejectionOO,Tukey1960ASO,Huber1964RobustEO}. 
The median-of-means method \cite{Lerasle2011ROBUSTEM,Minsker2015GeometricMA,Minsker2019DistributedSE} partitions the data into subsets, computes an estimate (such as the mean) from each, and takes their median. 
Similarly, some works study robustness from the optimization perspective, proposing to robustly aggregate gradients of the loss functions \cite{Su2016FaultTolerantMO,Su2016NonBayesianLI,Blanchard2017ByzantineTolerantML,Yin2018ByzantineRobustDL}.
We adapt some analysis techniques from \cite{Minsker2019DistributedSE,Yin2018ByzantineRobustDL} to our significantly different setting: estimation with heterogeneity and communication constraints.

\subsection{Notations}
Throughout the paper, for an integer $d\geq 1$, we write $[d]$ for both $\{1,\dots,d\}$ and $\{e_1,\dots,e_d\}\subseteq \RR^d$, where $e_k$ is the $k$-th canonical basis vector of $\RR^d$. For a vector $\bv\in\RR^d$, we refer to the entries of $\bv$ by both $[\bv]_1,\dots,[\bv]_d$ and $v_1,\dots,v_d$.
We denote 
$\|\bv\|_p = (\sum_{k\in[d]} |v_k|^p)^\frac{1}{p}$ for all $p>0$ with $\|\bv\|_0$ defined additionally as the number of non-zero entries. 
We let $\cP_d:= \{\bp = (p_1,\dots, p_d) \in [0, 1]^d: p_1 + \cdots + p_d = 1\}$ be the simplex of all $d$-dimensional discrete probability distributions. For $\bp\in\cP_d$, we denote by $\mathbb{B}_s(\bp)$ the $s$-distinct neighborhood $\{\bp^\prime\in\cP_d:\|\bp^\prime-\bp\|_0\leq s\}$.
For a random variable $X$, we denote $n$ i.i.d. copies of $X$  by $X^{[n]}$. 
Given any index set $\cI$, we write $|\cI|$ for its cardinality and denote by $[\bv]_{\cI}$ the sub-vector $([\bv]_k)_{k\in\cI}$ indexed by $\cI$. 
We use the Bachmann-Landau asymptotic notations $\Omega(\cdot)$, $\Theta(\cdot)$, $O(\cdot)$ to hide constant factors, and use $\tilde{\Omega}(\cdot)$, $\tilde{O}(\cdot)$ to also hide logarithmic factors. 
We denote the  categorical distribution with class probability vector $\bp\in\cP_d$ by $\Cat(\bp)$. In our algorithm to be introduced shortly, we use $\widecheck{\cdot}$ and $\widehat{\cdot}$ to indicate the intermediate estimate and the final estimate, respectively.

\section{Problem Setup}\label{sec:setup}
We consider the problem of collaboratively learning distributions defined according to the following model of heterogeneity (see Figure~\ref{fig:problem-structure} for an illustration).
There are $T \geq1$  clusters $\{\cC^t\triangleq (X^{t,j})_{j\in[n]}:t\in[T]\}$ of  user datapoints, each of which contains $n$ i.i.d. local datapoints. 
Each datapoint $X^{t,j}$ is in a one-hot format, \ie, $X^{t,j} \in \{e_1,e_2,\ldots,e_d\}$, and 
follows the categorical distribution $\Cat(\bp^t)$ where $\bp^t\in\cP_d$ is unknown. 
Thus, user datapoints in the same cluster $\cC^t$ have an identical distribution $\bp^t$, 
while the distribution $\bp^t$ can vary, \ie, be heterogeneous, across clusters $t\in[T]$. 
The datapoint $X^{t,j}$ is encoded by its user into a message $Y^{t,j}$, and then transmitted to a central server. We assume that the message sent by each datapoint is encoded into no more than $b$ bits and $b$ can be significantly smaller than $\log_2 d$ so that the communication is efficient.
We also assume the server knows which cluster $t \in [T]$ each $Y^{t, j}$ belongs to, as well as the number of clusters $T$.
This paper mainly addresses the communication bottleneck and does not involve privacy concerns. 

The goal here is to collaboratively learn the distributions $\bp^t$ from the collection of messages $\{Y^{t^\prime,[n]}\triangleq (Y^{t^\prime,j})_{j\in[n]}:t^\prime\in[T]\} $ despite heterogeneity. More precisely, we aim to design per-cluster estimators $\widehat{\bp}^t$, $t\in[T]$, with  $\widehat{\bp}^t: \{Y^{t^\prime,[n]}:t^\prime\in[T]\} \to \mathcal{P}_d$ to minimize the $\ell_2$ errors  
\begin{equation*}
    \E[\|\widehat{\bp}^t-{\bp}^t\|_2^2],\quad \text{for all $t\in[T]$}.
\end{equation*}
We also study the widely-used $\ell_1$ error metric (in addition to the $\ell_2$ metric).
When $T=1$, \ie, all user datapoints are homogeneous and there is a single distribution to learn, the problem reduces to the one studied by \cite{Han2018DistributedSE,Barnes2019LowerBF,Han2021GeometricLB,Acharya2020InferenceUI,Acharya2020GeneralLB}.

\begin{figure}
    \centering
    \includegraphics[width = 0.5
    \textwidth]{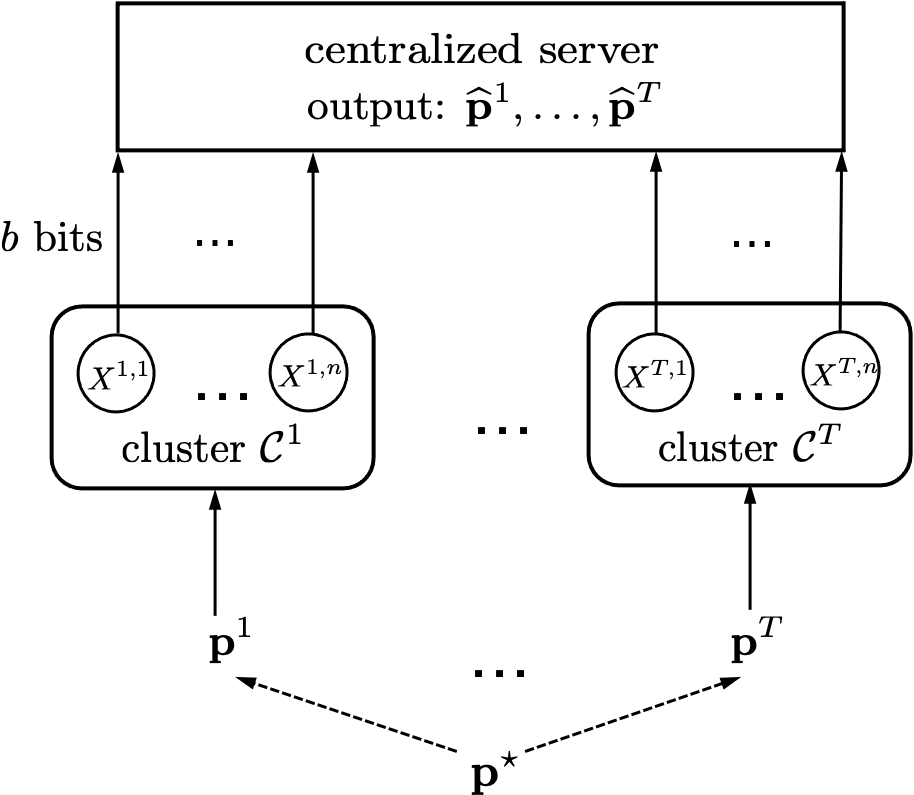}
    \caption{\small Learning distributions with heterogeneity and communication constraints.}
    \label{fig:problem-structure}
\end{figure}

\paragraph{Model of heterogeneity. }
In heterogeneous settings, 
collaboration among the users is most beneficial
if the local distributions are related. 
We model this by 
assuming that the local distributions are sparse perturbations of an unknown
\emph{central distribution} $\bp^\star\in \cP_d$. 
The distribution $\bp^t$ of each cluster $t$ differs from $\bp^\star$ in at most $s \geq0$ entries:
\begin{equation}\label{eqn:sparse-hetero}
    \|\bp^t-\bp^\star\|_0\leq s,\quad \forall\,t\in[T].
\end{equation}
The  central distribution $\bp^\star$ can be viewed as the central structure across heterogeneous clusters of datapoints. 
The level of heterogeneity is controlled by the parameter $s$. 
When $s$ is much smaller than $d$, the local distributions differ from the center in a small number of entries. 

While motivated by word frequencies of different texts, our model of sparse heterogeneity is also relevant for recommendation systems, where the high-dimensional item-preference vectors of users can vary sparsely \cite{Hu2019HERSMI,Qian2015StructuredSR,Liu2019RecommenderSW,Bastani2021PredictingWP}; and for medical risk scoring, where hospitals can have similar characteristics,  with a few systematic differences in diagnosis behavior, healthcare utilization, etc. \cite{Subbaswamy2019FromDT,QuioneroCandela2009DatasetSI,Mullainathan2017DoesML}.

\section{Algorithm}\label{sec:algorithm}
We now introduce our method for leveraging heterogeneous data to improve per-cluster estimation accuracy.  We first discuss a hashing-based method to handle the communication constraint. 
Since the communication between each user datapoint and the server is restricted to at most $b>0$ bits (where we may have $2^b\ll d$), 
the datapoint $X^{t,j}$ needs to be encoded by an encoding function $W^{t,j}:\mathcal{X}\triangleq [d]\rightarrow \mathcal{Y}$.
The $b$-bit constraint enforces $|\mathcal{Y}|\leq 2^b$.
Then the encoded message $Y^{t,j}:=W^{t,j}(X^{t,j})$ is sent to the server, where it is decoded and used.

There are relatively sophisticated communication protocols under which the design of encoding functions can be \emph{interactive}  \cite{Chen2021PointwiseBF,chen2021breaking,Acharya2020GeneralLB}, \ie, depend on previously sent messages. 
Here we adopt a \emph{non-interactive} encoding-decoding scheme, 
based on uniform hashing  \cite{pmlr-v132-acharya21b,chen2021breaking} where $W^{t,j}$ depends only on $X^{t,j}$ and is independent of other messages. 
Specifically, each datapoint $X^{t,j}$ is encoded  via an independent random hash function $h^{t,j}:[d]\rightarrow [2^b]$.
Upon receiving all messages, the server counts the empirical frequencies of all symbols, leading to hashed estimates $\widecheck{\bb}^t$. 
The communication scheme based on uniform hashing is summarized below.
\begin{align*}
    &{(\mathrm{Encoding}):\;}\text{Send the message $Y^{t,j}=h^{t,j}(X^{t,j})$ encoded by a hash function $h^{t,j}:[d]\rightarrow[2^b]$};\\
    &{(\mathrm{Decoding}):\;}\text{Count $N_k^t(Y^{t,[n]})=|\{j\in [n]:h^{t,j}(k)=Y^{t,j}\}|$ and return $[\widecheck{\bb}^t]_k=N_k^t/n$}.
\end{align*}
One can readily verify that  $\E[\widecheck{\bb}^t]=[(2^b-1)\bp^t+1]/2^b\triangleq \bb^t$; 
and thus the hashed estimate $\widecheck{\bb}^t$ is biased for $\bp^t$. We also write $\bb^\star=[(2^b-1)\bp^\star+1]/2^b$ for the mean of a hashed datapoint sampled from the central distribution.
More details on the hashed estimator $\widecheck{\bb}^t$ are given in Appendix \ref{app:uniform-hash}.

\subsection{The SHIFT Method}
We now introduce the SHIFT method, which consists of two stages: \emph{collaborative learning} and \emph{fine-tuning}.
The first stage estimates the central hashed distribution $\bb^\star$ using all hashed estimates $\{\widecheck{\bb}^t:t\in[T]\}$. 
This is achieved via methods from robust statistics such as the median or trimmed mean. 
The key insight here is that, since the heterogeneity is sparse, for each entry where the individual distributions mostly match with the central one, most datapoints (used to estimate that entry) are sampled from the probability of the central distribution. 
Hence, to estimate those entries of the central distribution, we can treat the datapoints generated by heterogeneous users as outliers, and leverage robust statistical methods to mitigate their influence.

\begin{algorithm}[t]
\caption{SHIFT: Sparse Heterogeneity Inspired collaboration
and Fine-Tuning}\label{alg:meta-method}
\begin{algorithmic}
\STATE \textbf{input:} individual hashed estimators $\widecheck{\bb}^1,\dots,\widecheck{\bb}^T$, threshold parameter $\alpha$
\vspace{1mm}
\STATE $\triangleright$ Stage I: Collaborative Learning
\vspace{1mm}
\STATE Estimate $\bb^\star$ via robust statistical methods: $\widecheck{\bb}^\star\,\leftarrow\,\mathrm{robust\_estimate}(\{\widecheck{\bb}^t:t\in[T]\})$
\vspace{1mm}
\STATE $\triangleright$ Stage II: Fine-Tuning
\vspace{1mm}
\FOR{$k=1,\dots,d$}
\FOR{$t = 1,\dots,T$}
\STATE $[\widehat{\bb}^t]_k\,\leftarrow\,[\widecheck{\bb}^\star]_k$\textbf{\; if \;}$|[\widecheck{\bb}^\star]_k-[\widecheck{\bb}^t]_k|\leq \sqrt{{\alpha[\widecheck{\bb}^t]_k}/{n}}$\textbf{\; else \;}$[\widecheck{\bb}^t]_k$
\STATE $[\widehat{\bp}^t]_k\,\leftarrow \,\proj_{[0,1]}(\frac{2^b [\widehat{\bb}^t]_k-1}{2^b-1})$  
\ENDFOR
\ENDFOR
\STATE \textbf{output:} estimates $\widehat{\bp}^1,\dots, \widehat{\bp}^T$
\end{algorithmic}
\end{algorithm}

In the second stage---fine-tuning---we detect mismatched entries between individual hashed estimates $\widecheck{\bb}^t$ and the central estimate $\widecheck{\bb}^\star$. 
Recall that the central and individual distributions differ in only a few entries.
For entries $k\in[d]$ such that $|[\widecheck{\bb}^\star]_k-[\widecheck{\bb}^t]_k|$ is below
$({\alpha[\widecheck{\bb}^\star]_k}/{n})^{1/2}$ for some threshold parameter $\alpha$, we may expect that $p^t_k=p_k^\star$. 
As a result, we expect the estimate $[\widecheck{\bb}^\star]_k$ of $b^\star_k$ to be more accurate, as it is learned collaboratively using a larger sample size.
Thus, we assign $[\widecheck{\bb}^\star]_k$ as the final estimate $[\widehat{\bb}^t]_k$ of $b^\star_k$.

On the other hand, for the entries where the central and individual distributions differ, \ie, $p^t_k\neq p_k^\star$,  the threshold is more likely to be exceeded. 
In this case, 
we keep the individual estimate $[\widecheck{\bb}^t]_k$ as $[\widehat{\bb}^t]_k$. 
Finally, since the hashed distributions $\bb^t$ are biased, we debias them in the final estimates of $\bp^t$ where $\proj_{[0,1]}(\cdot)$ in Algorithm \eqref{alg:meta-method} indicates truncating the input to the $[0,1]$ interval. 
Our method does not require sample splitting, despite using two stages, leading to increased sample-efficiency.
We remark that the SHIFT does not require the knowledge of the sparse heterogeneity $s$.

\paragraph{Knowledge transfer to new clusters. } The collaboratively learned central distribution from Algorithm \ref{alg:meta-method} is adaptable to new clusters, which possibly only have a few datapoints. 
This can be particularly beneficial for sample efficiency, because most entries of the target distribution are well-estimated through collaborative learning. 
One can transfer those entries, and it suffices to estimate the few remaining entries, instead of the whole distribution.
See Theorem \ref{thm:transfer} for the details. 
The knowledge transfer utility further motivates the importance of collaborative learning.

\subsection{Median-Based SHIFT}
In this section, we provide upper bounds on the error for the median-based SHIFT method, where $\mathrm{robust\_estimate}(\{\widecheck{\bb}^t:t\in[T]\})$ in Algorithm \ref{alg:meta-method} is the entry-wise median.
Specifically, we let 
\begin{equation*}
    [\widecheck{\bb}^\star]_k =\med\big(\{[\widecheck{\bb}^t]_k:t\in[T]\}\big),\quad \text{for each $k\in[d]$}.
\end{equation*}
When there is no ambiguity, we write $\widecheck{\bb}^\star =\med\big(\{\widecheck{\bb}^t\}_{t\in[T]}\big)$.
We also provide results for the trimmed-mean-based SHIFT method, see Appendix \ref{app:trimmed-mean}.

By setting the threshold parameter $\alpha$ in Algorithm \ref{alg:meta-method} as  $\alpha =\Theta(\ln(n))$, we prove the following upper bounds on the final individual $\ell_2$ estimation errors. 
The results for the $\ell_1$ error are  in Appendix \ref{app:median}.
\begin{theorem}\label{thm:median-collab-final}
Suppose  $n \geq 2^{b + 6}\ln(n)$ and $\alpha =\Theta(\ln(n))$\footnote{To be precise, we require $\alpha=O(\ln(n))$ and $\alpha \geq c\ln(n)$ for some absolute constant $c$. The analogous statement applies in Theorem \ref{thm:transfer}.}. Then, for the median-based SHIFT method, for any $t\in[T]$,
\begin{equation*}\label{eqn:median-final-knowledge}
\setlength{\abovedisplayskip}{3pt}\setlength{\belowdisplayskip}{1pt}
 \E\left[\|\widehat{\bp}^t-\bp^t\|_2^2\right]=\tilde{O}\left(\frac{\max\{2^b,s\}}{2^bn}+\frac{d}{2^bTn}+\frac{d}{n^2}\right).
\end{equation*}
\end{theorem}
When $n\geq 2^{b+6}\ln(n) \text{ and }
n=\Omega( 2^b\min\{T, d/\max\{2^b,s\}\})$, the rate further becomes
\begin{equation}\label{eqn:ghoqenfq}
\setlength{\abovedisplayskip}{3pt}\setlength{\belowdisplayskip}{3pt}
 \E\left[\|\widehat{\bp}^t-\bp^t\|_2^2\right]=\tilde{O}\left(\frac{\max\{2^b,s\}}{2^bn}+\frac{d}{2^bTn}\right).
\end{equation}
The upper bound in \eqref{eqn:ghoqenfq} consists of two terms. The first term---$\max\{2^b,s\}/(2^bn)$---is independent of the dimension $d$, and is smaller than the rate $d/(2^bn)$ obtained by the minimax optimal method using only homogeneous datapoints \cite{Han2018DistributedSE} by a factor ${d}/{\max\{2^b,s\}}$.
Thus, it brings a significant benefit under sparse heterogeneity and reasonable communication restrictions, \ie, when $\max\{2^b,s\}\ll d$. 
Meanwhile, the second, dimension-dependent, term $d/(2^bTn)$ is $T$ times smaller than  $d/(2^bn)$, since it depends on the \emph{total sample-size} $Tn$ used collaboratively, despite heterogeneity. Therefore, our method shows a factor of $\min\{T,{d}/{\max\{2^b,s\}}\}$ improvement in sample efficiency, compared to previous work designed for homogeneous datapoints.

For completeness, we also consider a heuristic application of  estimators from prior works  \cite{Han2018DistributedSE,Chen2021PointwiseBF}, in which datapoints from all clusters are pooled to learn a global distribution $T^{-1}\sum_{t\in[T]}\bp^t$, which is then used by each cluster. 
While this uses all datapoints, it inevitably introduces a non-vanishing bias $\boldsymbol{\delta}^t=\bp^t-T^{-1}\sum_{t^\prime\in[T]}\bp^{t^\prime}$ in estimating individual distributions, and can behave poorly when the bias is large. See Table \ref{tab:comparsion} for more details.

Finally, we discuss our results on knowledge transfer. The central estimator $\widecheck{\bb}^\star$ is adaptable to a new cluster $\cC^{T+1}$ of size $\tilde{n}$ in the following way. 
We adjust the fine-tuning procedure in Algorithm \ref{alg:meta-method} to $[\widehat{\bb}^{T+1}]_k\,\leftarrow\,[\widecheck{\bb}^\star]_k$
if
$|[\widecheck{\bb}^\star]_k-[\widecheck{\bb}^{T+1}]_k|\leq \sqrt{{\alpha[\widecheck{\bb}^{T+1}]_k}/{\tilde{n}}}$,
and
$[\widehat{\bb}^{T+1}]_k\,\leftarrow\, [\widecheck{\bb}^{T+1}]_k$
otherwise. 
We then show the following result.
\begin{theorem}\label{thm:transfer}
Let $\widecheck{\bb}^{T+1}$ be the hashed estimate of any new cluster $\cC^{T+1}$ with $\tilde{n}$ datapoints such that $n\geq \tilde{n}\geq 2^{b+6} \ln(\tilde{n})$ and $\tilde{n}=\Omega(2^b \min\{T,d/\max\{2^b,s\}\})$.
Let the threshold parameter be $\alpha =\Theta(\ln( \tilde{n}))$. Then,
the median-based SHIFT  method has error bounded by
\begin{equation*}\label{eqn:median-final-transfer}
\setlength{\abovedisplayskip}{3pt}\setlength{\belowdisplayskip}{2pt}
 \E\left[\|\widehat{\bp}^{T+1}-\bp^{T+1}\|_2^2\right]=\tilde{O}\left(\frac{\max\{2^b,s\}}{2^b\tilde{n}}+\frac{d}{2^bTn}\right).
\end{equation*}
\end{theorem}
Similarly, one can see that the adaptation to new clusters with the median-based SHIFT method achieves a factor of $\min\{Tn/\tilde{n},d/\max\{2^b,s\}\}$ improvement in sample-efficiency compared to estimating the distribution of the new cluster without knowledge transfer.

\subsubsection{Highlights of Theoretical Analysis}
In this section, we introduce the key ideas behind the proof of Theorems \ref{thm:median-collab-final} and \ref{thm:transfer}. Our analysis is novel compared to previous analyses for methods with homogeneous datapoints.
The final individual estimation errors relate to the error of estimating the central hashed distribution $\bb^\star$. 
However, we only expect high accuracy at the center for entries with few individual misalignments. 
To quantify the influence of heterogeneity, for any $0< \eta \leq 1$, we define the set
of \emph{$\eta$-well-aligned entries} as 
\begin{equation*}
    \cI_\eta:=\{k\in[d]:|\cB_k|< \eta T\},\quad \text{where}\quad \cB_k\triangleq \{t\in[T]:b_k^t\neq b_k^\star,\,\ie,\,p_k^t\neq p_k^\star\}
\end{equation*}
is the set of clusters whose distribution differs from $\bp^\star$ in the $k$-th entry.
We aim to estimate the $\eta$-well-aligned entries accurately by using robust statistical methods.

Further, we argue that there are few poorly-aligned entries, 
and they affect the final per-cluster error only mildly.
By the pigeonhole principle, the number of entries that are not $\eta$-well-aligned is upper bounded by $|\cI_\eta^c|\triangleq |[d]\backslash\cI_\eta|\leq \frac{s T}{\eta T}=  \frac{s}{\eta}.$
Therefore, given an estimator $\widecheck{\bb}^\star$ that is accurate for the $\eta$-well-aligned entries, 
the entries of  $\bb^\star$ can be estimated accurately except for at most $s/\eta$ entries. 
The following technical lemma bounds the error for each entry $k\in\cI_\eta$.
\begin{lemma}\label{lem:median-entry}
Suppose $\widecheck{\bb}^\star =\med\big(\{\widecheck{\bb}^t\}_{t\in[T]}\big)$. Then for any  $0< \eta \leq  1/5$ and $k\in\cI_\eta$,  it holds that 
\begin{equation*}
\setlength{\abovedisplayskip}{4pt}
    \E[([\widecheck{\bb}^\star]_k-[{\bb}^\star]_k)^2]=\tilde{O}\left( \frac{|\cB_k|^2b_k^\star(1-b_k^\star)}{T^2n}+\frac{b_k^\star(1-b_k^\star)}{Tn}+\frac{1}{n^2}\right).
    \setlength{\belowdisplayskip}{1pt}
\end{equation*}
\end{lemma}
Lemma \ref{lem:median-entry} provides an upper bound, in terms of the frequency ${|\cB_k|}/{T}$ of misalignment (which is smaller than $\eta$), and a variance term $b_k^\star(1-b_k^\star)$. 
This result cannot be obtained by directly applying the standard Chernoff or Hoeffding bounds to random variables distributed in $[0,1]$ as in previous works \cite{Chen2021PointwiseBF} for two reasons: 1) the datapoints are heterogeneous, 2) the variance $b_k^\star(1-b_k^\star)$ here can be small, compared to what it can be for general random variables in $[0,1]$, implying more concentration than what follows from Hoeffding's inequality.
To address these issues, we analyze the concentration of the empirical $(1/2\pm{|\cB_k|}/{T})$-quantiles to mitigate the influence of heterogeneity, and we also use Bernstein's inequality, which is variance-dependent \cite{Uspensky1937IntroductionTM}, to obtain bounds relying on both the sample size $Tn$ and the variance $b_k^\star(1-b_k^\star)$.

Also, the constant $1/5$, upper bounding the  heterogeneity, is not essential and is chosen for clarity. 
It can be replaced with any number less than half; in which case estimating the central probability distribution becomes possible, as the information conveyed by homogeneous datapoints dominates.

Lemma  \ref{lem:median-entry} reveals that well-aligned entries of the central distribution are accurately estimated. Thus one can use the central estimate for the entries where the central distribution $\bp^\star$ aligns with the target distribution $\bp^t$. 
The remaining entries, that are neither well-aligned nor satisfy $p_k^\star=p_k^t$, can be estimated  by the local estimator. 
We argue that a properly chosen threshold parameter $\alpha$ filters out the only desired entries to be estimated individually, with high probability, leading to Theorems \ref{thm:median-collab-final} and \ref{thm:transfer}.

While estimating $\bp^\star$ is not our main goal, one can readily obtain from Lemma  \ref{lem:median-entry} the following bound for estimating $\bp^\star$ by summing up the errors for all entries $k\in[d]=\cI_{\eta}$ with $\eta =\max_{k\in[d]}|\cB_k|/T$.
Corollary \ref{cor:median-uniform-hetero} reveals that the central distribution can be accurately estimated if the  mismatch of distributions happens uniformly across all entries, \ie, each entry differs in  $O(sT/d)$  clusters.

\begin{corollary}\label{cor:median-uniform-hetero}
 Let $\widehat{\bp}^\star=\proj_{[0,1]}(\frac{2^b \widecheck{\bb}^\star-1}{2^b-1})$ be obtained by the debiasing operation from Algorithm \ref{alg:meta-method}.
Suppose $|\cB_k|=O(sT/d)$ for any $k\in [d]$, with $\eta =\max_{k\in[d]}|\cB_k|/T$. 
Then the median-based SHIFT method enjoys
\begin{equation*}
    \E[\|\widehat{\bp}^\star-\bp^\star\|_2^2]=\tilde{O}\left(\frac{s^2}{d2^bn}+\frac{d}{2^bTn}+\frac{d}{n^2}\right).
\end{equation*}
\end{corollary}

\section{Lower Bounds}\label{sec:lower-bounds}
To complement our upper bounds, we now provide 
minimax lower bounds for estimating distributions under heterogeneity. Since our setting contains $T$ heterogeneous clusters of datapoints, 
our minimax error metric is slightly different from the one studied in \cite{Han2018DistributedSE,Barnes2019LowerBF,Han2021GeometricLB,Acharya2020InferenceUI}. Using the $\ell_2$ error as the loss, the lower bound metric is defined as 
\begin{equation}\label{eqn:vhoeqwfq}
   \inf_{\substack{(W^{t^\prime,[n]})_{t^\prime \in[T]}\\\widehat{\bp}^{t} }}\sup_{\substack{\bp^\star \in\cP_d\\ \{\bp^{t^\prime}:t^\prime\in[T]\}\subseteq\mathbb{B}_s(\bp^\star)}}
    \E\left[\left\|\widehat{\bp}^t-\bp^t\right\|_2^2\right],
\end{equation}
where the supremum is taken over all possible central distributions $\bp^\star\in \cP_d$ and individual distributions $\{\bp^t:t\in[T]\}$ in  $\mathbb{B}_s(\bp^\star)\triangleq \{\bp\in\cP_d:\|\bp-\bp^\star\|_0\leq s\}$, and the infimum is taken over all estimation methods $\hat\bp^t$ that use all heterogeneous messages $\{Y^{t^\prime,j}\triangleq W^{t^\prime,j}(X^{t^\prime,j}):j\in[n], t^\prime\in[T]\}$ encoded (possibly interactively) by any encoding functions
$\{W^{t^\prime,j}:j\in[n], t^\prime\in [T] \}$ with output in $[2^b]$, \eg, the random hashing maps.  The measure \eqref{eqn:vhoeqwfq} characterizes the best possible worst-case performance of estimating distributions under our model of heterogeneity.

Since the supremum is taken over all  distributions $\bp^\star,\bp^1,\dots,\bp^T$ in $\cP_d$ such that $\|\bp^t-\bp^\star\|_0\leq s$ for all $t\in[T]$, we  consider  two representative cases therein: 
First, in the \emph{homogeneous} case where $\bp^1=\cdots=\bp^T=\bp^\star \in\cP_d$,  the setting reduces to the single-cluster problem but with $nT$ datapoints, and with the goal of estimating $\bp^\star$, leading to the lower bound $\Omega(d/(2^bTn))$. 
Second, in the \emph{$s/2$-sparse} case where  $\|\bp^\star\|_0\leq s/2$ and $\|\bp^t\|_0\leq s/2$ for all $t\in[T]$, we have $\{\bp^t:t\in[T]\}\subseteq\mathbb{B}_s(\bp^\star)$. By constructing independent priors for $\{\bp^t:t\in[T]\}$ and $\bp^\star$, one can show that only datapoints generated by $\bp^t$ itself are informative for estimating $\bp^t$. In this case, we show the lower bound $\Omega(\max\{2^b,s\}/(2^bn))$. Combining the two cases, we find the following lower bound. The formal argument is provided in Appendix \ref{app:lower-bounds}.
\begin{theorem}\label{thm:collab-lower-bound}
For any---possibly interactive---estimation method, and for any $t\in[T]$, we have
\begin{equation}\label{eqn:collab-lower-bound}
\setlength{\abovedisplayskip}{3pt}\setlength{\belowdisplayskip}{0.5pt}
    \inf_{\substack{(W^{t^\prime,[n]})_{t^\prime \in[T]}\\ \widehat{\bp}^{t}\, }} 
    \sup_{\substack{\bp^\star \in\cP_d\\  \{\bp^{t^\prime}:t^\prime\in[T]\}\subseteq\mathbb{B}_s(\bp^\star)}}
    \E[\|\widehat{\bp}^t-\bp^t\|_2^2]=\Omega\left(\frac{\max\{2^b,s\}}{2^bn}+\frac{d}{2^bTn}\right).
\end{equation}
\end{theorem}

By a similar argument but with an additional $(T+1)$-st cluster of $\tilde{n}$ users, we obtain a lower bound for adapting to a new cluster.  
\begin{theorem}\label{thm:transfer-lower-bound}
For any---possibly interactive---estimation method, and a new cluster $\cC^{T+1}$, we have
\begin{equation}\label{eqn:transfer-lower-bound}
\setlength{\abovedisplayskip}{3pt}\setlength{\belowdisplayskip}{1pt}
    \inf_{\substack{(W^{t^\prime,[n]})_{t^\prime \in[T]}\\W^{T+1,[\tilde{n}]},\widehat{\bp}^{T+1} }}  \sup_{\substack{\bp^\star \in\cP_d\\ \{\bp^{t^\prime}:t^\prime\in[T+1]\}\subseteq\mathbb{B}_s(\bp^\star)}}\E[\|\widehat{\bp}^{T+1}-\bp^{T+1}\|_2^2]=\Omega\left(\frac{\max\{2^b,s\}}{2^b\tilde{n}}+\frac{d}{2^bTn}\right).
\end{equation}
\end{theorem}
Theorem \ref{thm:collab-lower-bound} and \ref{thm:transfer-lower-bound}, combined with the upper bounds in Section \ref{sec:algorithm}, imply that our method is minimax optimal up to logarithmic terms. 
We provide similar lower bounds for the $\ell_1$ error in Appendix \ref{app:lower-bounds}.

\section{Experiments}\label{sec:experi}
We test SHIFT on synthetic data as well as on the Shakespeare dataset \cite{caldas2018leaf}.
As a baseline method, we use the estimator in \cite{Han2018DistributedSE} that is minimax optimal under homogeneity.
We apply the baseline method both locally and globally.
In the local case, the estimator $\widehat{\mathbf{p}}^t$ for each cluster is computed without datapoints from other clusters.
In the global case, we pool data from all clusters, 
and compute estimators $\widehat{\mathbf{p}} = \widehat{\mathbf{p}}^1 = \cdots = \widehat{\mathbf{p}}^T$. 
The performance measure for estimating $\widehat{\mathbf{p}}^t$, $t \in [T]$ is taken as $T^{-1}\sum_{t = 1}^T \Vert \mathbf{p}^t - \widehat{\mathbf{p}}^t \Vert_2^2.$

\subsection{Synthetic Data}\label{sec:synthetic}
We set the uniform distribution, $\mathbf{p}^\star = (1/d, \dots, {1}/{d})$ as the central distribution.
In Appendix \ref{app:experiments}, we also experiment on the truncated geometric distribution, and compare our method with the localization-refinement method \cite{Chen2021PointwiseBF}.
Among the $d$ entries of $\mathbf{p}^\star$, we draw $s$ entries uniformly at random 
and assign new values for them uniformly at random over $[0,1]$, with re-normalization to preserve their sum.
We repeat this procedure $T$ times to obtain sparsely perturbed distributions $\mathbf{p}^1, \dots, \mathbf{p}^T \in \mathcal{P}_d$.
Then, $n$ i.i.d. datapoints $X^{t, 1}, \dots, X^{t, n} \sim \Cat(\mathbf{p}^t)$ are generated for each cluster $t \in [T]$.
We set the dimension to $d = 300$ and run the simulation by varying $n, T, s$.
As we see from \eqref{eqn:ghoqenfq}, the error of our method depends on $s$ only when $2^b < s$.
For this reason, we let $b = 2$ in our experiments.

\begin{figure}[t]
  \centering
  \includegraphics{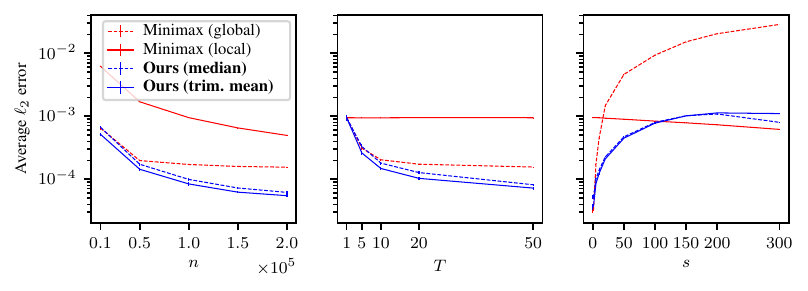}
  \caption{Average $\ell_2$ estimation error in synthetic experiment. (Left):  Fixing $s = 5, \,T = 30$ and varying $n$.
  (Middle): Fixing $s = 5,\,n = 100,000$ and varying $T$.
  (Right): Fixing  $T = 30, \,n = 100,000$ and varying $s$.
  The standard error bars are obtained from 10 independent runs.}
  \label{fig:synthetic}
\end{figure}

We run SHIFT with the entry-wise median and entry-wise trimmed mean as the robust estimate.
We set the threshold parameter $\alpha = \ln(n)$ and the trimming proportion $\omega = 0.1$.
In Appendix \ref{app:experiments}, we provide results for other choices of the hyperparameters $\alpha, \omega$ and discuss a heuristic for choosing $\alpha$.
Figure~\ref{fig:synthetic} illustrates that our method outperforms the baseline method for most choices of $n, T, s$.
Specifically, as Theorem \ref{thm:median-collab-final} reveals, the $\ell_2$ error of our method decreases as the number of clusters $T$ increases.
On the other hand, when the baseline methods are applied globally, without considering heterogeneity, 
they show a bias that does not disappear as the sample size $n$ or the number of clusters $T$ increases.
This shows that the fine-tuning step in SHIFT is crucial for the estimation of heterogeneous distributions.
Finally, the right panel of Figure~\ref{fig:synthetic} shows that our method is effective only when $s$ is small compared to the dimension $d$; which highlights the crucial role of sparse heterogeneity.
When $s$ is close to $d$, the distributions $\mathbf{p}^1, \dots \mathbf{p}^T$ could be considerably different without any meaningful central structure, making collaboration less useful than local estimation.

\subsection{Shakespeare Dataset}
The Shakespeare dataset was proposed as a benchmark for federated learning in \cite{caldas2018leaf}.
The dataset consists of dialogues of 1,129 speaking roles in Shakespeare's 35 different plays.
In our experiment, we study the distribution of $k$-grams ($k$-tuples of consecutive letters from the 26-letter English alphabet, see Chapter 3 of \cite{leskovec2020mining}) appearing in the dialogues.
We consider each play as a cluster $\mathcal{C}^t$ and estimate the distribution $\mathbf{p}^t \in \mathcal{P}_{d}$, $d=26^k$ of $k$-grams.
Since the ground-truth distribution $\mathbf{p}^t$ is unknown, we regard the empirical frequency as $\mathbf{p}^t$.

\begin{table}[h]
\centering
\begin{tabular}{lcccc} 
\toprule
\multicolumn{1}{c}{$k = 2$} & $b = 2$       & $b = 4$       & $b = 6$        & $b = 8$         \\ 
\hline
Minimax (local)             & $640 \pm 6.0$ & $142 \pm 1.2$ & $40 \pm 0.40$  & $14 \pm 0.13$   \\
Minimax (global)            & $33 \pm 1.8$  & $17 \pm 0.37$ & $14 \pm 0.081$ & $13 \pm 0.037$  \\
SHIFT (median)               & $47 \pm 2.4$  & $21 \pm 0.66$ & $14 \pm 0.17$  & $11 \pm 0.10$   \\
SHIFT (trimmed mean)         & $36 \pm 2.2$  & $19 \pm 0.51$ & $13 \pm 0.24$  & $10 \pm 0.062$  \\ 
\hline
\multicolumn{1}{c}{$k = 3$} & $b = 2$       & $b = 4$       & $b = 6$        & $b = 8$         \\ 
\hline
Minimax (local)             & $15000 \pm 21$ & $3000 \pm 5.9$ & $720 \pm 2.1$  & $180 \pm 0.39$   \\
Minimax (global)           & $4400 \pm 5.7$  & $100 \pm 1.4$ & $ 38 \pm 0.35$ & $23 \pm 0.090$  \\
SHIFT (median)               & $7300 \pm 9.6$  & $180 \pm 2.1$ & $ 53 \pm 1.0$  & $20 \pm 0.18$   \\
SHIFT (trimmed mean)         & $5100 \pm 6.3$  & $140 \pm 2.3$ & $ 43 \pm 0.66$  & $ 18 \pm 0.18$  \\
\bottomrule
\end{tabular}
\caption{Average $\ell_2$ error for estimating distributions of $k$-grams in the Shakespeare dataset. Numbers are scaled by $10^{-5}$.}
\label{tbl:res}
\end{table}

To verify the heterogeneity, we run the chi-squared goodness-of-fit test for each pair of distributions from distinct clusters $\mathbf{p}^u$ and $\mathbf{p}^v$.
Resulting p-values are essentially zero within machine precision, which suggests that the distributions of $k$-grams are strongly heterogeneous.
We also perform entry-wise tests comparing $[\mathbf{p}^u]_i$ and $[\mathbf{p}^v]_i$ for all $u \neq v \in [T], i \in [d]$.
In total, 25.8\% of the tests are rejected at the  5\% level of significance.
This is again consistent with heterogeneity.

We draw  $n = 20,000$ datapoints with replacement from each cluster  and test SHIFT with communicati-on budgets $b \in \{2, 4, 6, 8\}$.
We set the fine-tuning threshold $\alpha = \ln(n)$ and the trimming proportion $\omega = 0.1$, which we choose following the heuristic discussed in Appendix \ref{app:experiments}.
We repeat the experiment ten times by randomly drawing different datapoints, and report the average $\ell_2$ error of estimation in Table \ref{tbl:res}.
The standard deviations are small even over ten repetitions.
SHIFT shows competitive performance on the empirical dataset, even though we do not rigorously know if the sparse heterogeneity model \eqref{eqn:sparse-hetero} applies.

\section{Conclusion and Future Directions}
We formulate the problem of learning distributions under sparse heterogeneity and communication constraints. We propose the SHIFT method, which first learns a central distribution, and then fine-tunes the estimate to adapt to individual distributions. 
We provide both theoretical and experimental results to show  its sample-efficiency improvement compared to classical methods that target only homogeneous data. 
Many interesting directions remain to be explored, including investigating if there is a point-wise optimal method with rate depending on $\{\bp^t:t\in[T]\}$ and $\bp^\star$; and designing methods for other information
constraints, such as local differential privacy constraints.

\section*{Acknowledgements}
During this work, Xinmeng Huang was supported in part by the NSF TRIPODS 1934960, NSF DMS 2046874 (CAREER), NSF CAREER award CIF-1943064; Donghwan Lee was supported in part by ARO W911NF-20-1-0080, DCIST, Air Force Office of Scientific Research Young Investigator Program (AFOSR-YIP)
award
\#FA9550-20-1-0111.

\bibliography{reference}
\bibliographystyle{abbrv}

\section{Appendix}
\paragraph{Additional Notations. }
In the appendix, we use the following additional notations.
For an integer $d\geq 1$, and a vector $\bv\in\RR^d$, the support $\supp(\bv) = \{j\in[d]: \bv_j\neq 0\}$ denotes the indices of non-zero entries.
For an event $A$ on a probability space $(\Omega,B,P)$ (which is usually self-understood from the context), we denote by $I(A)$, $\mathds{1}\{A\}$, or $\mathds{1}(A)$ its indicator function, such that $I(A)(\omega)=1$ if $\omega\in A$, and zero otherwise.
We denote by $\Phi$ the cumulative distribution function of the standard normal random
variable.
For two scalars $a,b\in\R$, we write $a\wedge b = \min(a,b)$.

\section{Properties of Uniform Hashing}\label{app:uniform-hash}

\begin{algorithm}[h]
\caption{Encoding-Decoding via Uniform Hashing }\label{alg:hashed-estimate}
\begin{algorithmic}
\STATE \textbf{input:} cluster $\cC^t$ with $n\ge1$ users having data $X^{t,j}$, $j=1,\ldots,n$
\FOR{$j = 1,\dots,n$}
\STATE Generate a uniformly random hash function $h^{t,j}:[d]\rightarrow[2^b]$ using shared randomness 
\STATE Encode message $Y^{t,j}=h^{t,j}(X^{t,j})$ and send it to the server \hspace{25mm}$\triangleright$ Encoding
\ENDFOR
\FOR{$k = 1,\dots,d$}
\STATE Count $N_k^t(Y^{t,[n]})\,\leftarrow \, |\{j\in [n]:h^{t,j}(k)=Y^{t,j}\}|$ \hspace{43mm}$\triangleright$ Decoding
\STATE Estimate $\widecheck{b}^t_k\,\leftarrow \,N_k^t/n$
\ENDFOR
\STATE \textbf{output:}  $\widehat{\bb}^t$
\end{algorithmic}
\end{algorithm}

Recall that for all $t\in[T]$ and $k\in[d]$, $b_k^t=\frac{p_k^t(2^b-1)+1}{2^b}\in\left[\frac{1}{2^b},1\right]$.

\begin{proposition}[Properties of Hashed Estimates]\label{lem:uniform-hashing-basic}
For each $t\in[T]$, 
suppose $\widecheck{\bb}^t$ is computed in cluster $\cC^t$ as in Algorithm \ref{alg:hashed-estimate} with i.i.d datapoints $X^{t,j}\sim \Cat(\bp^t)$, $\forall\,j\in[n]$.
Then, it holds that
\begin{enumerate}
    \item $\widecheck{\bb}^1,\dots,\widecheck{\bb}^T\in[0,1]$ are independent;
    \item for any $t\in[T]$ and $k\in[d]$, $N_k^t\sim \Bin(n,b^t_k)$;
    \item  $\supp(\bp^t-\bp^\star)=\supp(\bb^t-\bb^\star)$ and $p_k^\star=1$ (or $0$) is equivalent to $b_k^\star=1$ (or $\frac{1}{2^b}$, respectively).
\end{enumerate}
\end{proposition}
\begin{proof}
Property 1 holds because $\widehat{\bb}^1,\dots,\widehat{\bb}^T$ are obtained by cluster-wise encoding-decoding of independent datapoints. 
To see property 2, 
we have for any $j\in[n]$ and $k\in[d]$ that
\begin{align*}
    \PP(h^{t,j}(k)=Y^{t,j})&=\PP(k=X^{t,j})+\PP(k\neq X^{t,j}\text{ and }h^{t,j}(k)=h^{t,j}(X^{t,j}))\\
    &=
    p^t_k+(1-p^t_k)\cdot \frac{1}{2^b}= b^t_k\in\left[\frac{1}{2^b},1\right].
\end{align*}
Thus, $I(h^{t,j}(k)=Y^{t,j})$ is a $\mathrm{Bernoulli}$ variable with success probability $b^t_k$.
Since each datapoint is encoded with an independent hash function,
$N_k^t$ has a binomial distribution with $n$ trials and parameter $b^t_k$. 
Property 3 directly follows from $\bb^t-\bb^\star=(\bp^t-\bp^\star)(2^b-1)/2^b$ and as $b>0$.
\end{proof}
\begin{proposition}[Property of Debiasing]\label{prop:debias}
For any $\by,\by^\star \in\RR^d$, let $\bx=\proj_{[0,1]}(\frac{2^b \by-1}{2^b-1})$ and $\bx^\star=\proj_{[0,1]}(\frac{2^b \by^\star-1}{2^b-1})$. 
Then it holds that  for $q=1,2$, $\E[\|\bx-\bx^\star\|_q^q]=O( \E[\|\by-\by^\star\|_q^q])$. 
In particular, we have  for $q=1,2$ and any $t\in[T]$, $\E[\|\widehat{\bp}^t-\bp^t\|_q^q]=O( \E[\|\widehat{\bb}^t-\bb^t\|_q^q])$, where $\widehat{\bp}^t=\proj_{[0,1]}(\frac{2^b \widehat{\bb}^t-1}{2^b-1})$ is the final per-cluster estimate obtained in Algorithm~\ref{alg:meta-method}.
\end{proposition}
\begin{proof}
Using the inequality that $|\proj_{[0,1]}(x)-\proj_{[0,1]}(y)|\leq |x-y|$ for any $x,y\in\RR$, we have
\begin{align*}
    \E[\|\bx-\bx^\star\|_q^q]&=\sum_{k\in[d]}\E\left[\left|\proj_{[0,1]}\left(\frac{2^b y_k-1}{2^b-1}\right)-\proj_{[0,1]}\left(\frac{2^b y^\star_k-1}{2^b-1}\right)\right|^q\right]\\
    &\leq \sum_{k\in[d]}\E\left[\left|\frac{2^b( y_k-y^\star_k)}{2^b-1}\right|^q\right]= 
   \left(\frac{2^b}{2^b-1}\right)^q \E[\|\by-\by^\star\|_q^q]=O(\E[\|\by-\by^{\star}|_q^q]).
\end{align*}
In the last step, we used that $2^b/(2^b-1) \leq 2$ for all $b\geq1$, and thus the $O(\cdot)$ only depends on universal constants.
\end{proof}

\section{General Lemmas}
In this section, we state some general lemmas that will be used in the analysis.
\begin{lemma}[Berry-Esseen Theorem;  \cite{Vershynin2018HighDimensionalP}]\label{lem:berry}
Assume that $Z_{1}, \ldots, Z_{n}$ are $i . i . d .$ copies of a random variable $Z$ with mean $\mu$, variance $\sigma^{2}>0$, and such that $\mathbb{E}\left[|Z-\mu|^{3}\right]<\infty$. Then,
\[
\sup _{x \in \mathbb{R}}\left|\mathbb{P}\left\{\sqrt{n} \frac{\bar{Z}-\mu}{\sigma} \leq x\right\}-\Phi(x)\right| \leq0.4748 \frac{\gamma(Z)}{\sqrt{n}}.
\]
where $\bar{Z}=\frac{1}{n} \sum_{i=1}^{n} Z_{i}$ and $\gamma(Z)=\E[|Z-\mu|^3]/\sigma^3$ is the absolute skewness of $Z$.
\end{lemma}

\begin{lemma}[Hoeffding's Inequality; \cite{Vershynin2018HighDimensionalP}]\label{h}
Let $Z_{1}, \ldots, Z_{n}\in[l,r]$, $l<r$, be independent random variables and let $\bar{Z}=\frac{1}{n}\sum_{j=1}^nZ_j$. 
Then for any $\delta \geq 0$,
\[
\max\{\mathbb{P}(\bar{Z}-\mathbb{E}[\bar{Z}] > \delta ),\mathbb{P}(\bar{Z}-\mathbb{E}[\bar{Z}] <- \delta)\} \leq \exp \left(-\frac{2n \delta^{2}}{(r-l)^2}\right).
\]
\end{lemma}

\begin{lemma}[Bernstein's Inequality; \cite{Uspensky1937IntroductionTM}]\label{b}
Let $Z_{1}, \ldots, Z_{n}$ be $i . i . d .$ copies of a random variable $Z$ with $|Z-\E[Z]|\leq M$, $M>0$ and $\mathrm{Var}(Z_1)=\sigma^2>0$,  and let $\bar{Z}=\frac{1}{n}\sum_{j=1}^nZ_j$. 
Then for any $\delta \geq 0$,
\begin{equation}\label{eqn:jgowmfqww}
    \mathbb{P}(|\bar{Z}-\mathbb{E}[\bar{Z}]| > \delta) \leq2\exp \left(-\frac{n \delta^{2}}{2(\sigma^2+M\delta)}\right)\leq 2\exp \left(-\frac{n}{4}\min\left\{\frac{\delta^2}{\sigma^2},\frac{\delta}{M}\right\}\right).
\end{equation}
\end{lemma}
The second inequality above directly follows from $\frac{1}{a+b}\geq \frac{1}{2}\min\{\frac{1}{a},\frac{1}{b}\}$ for any $a,b>0$. Note that \eqref{eqn:jgowmfqww} also allows $\sigma=0$ because $ \mathbb{P}(|\bar{Z}-\mathbb{E}[\bar{Z}]| > \delta)=0$ and $\min\left\{{\delta^2}/{\sigma^2}\triangleq+\infty,{\delta}/{M}\right\}=\frac{\delta}{M}$. Therefore, we use this lemma for all $\sigma\geq0$ below.

\subsection{Analysis Framework}
For each $t\in[T]$, we denote  by 
\begin{equation}\label{kt}
    {\cK}^t_\alpha=\{k\in[d]: (\widecheck{b}^\star_k-\widecheck{b}^{t}_k)^2\leq {\alpha\widecheck{b}^t_k}/{n}\}
\end{equation} 
the set of entries in which the central estimate $[\widehat{\bb}^\star]_k$ is adapted to cluster $\cC^t$.
In this language, the final estimates can be expressed as $\widehat{b}^t_k=\widecheck{b}^\star_k\mathds{1}\{k\in {\cK}^t_\alpha\}+ \widecheck{b}^t_k\mathds{1}\{k\notin {\cK}^t_\alpha\} $ for $t\in [T]$.
Therefore, it holds that, for $q=1,2$,
\begin{equation}\label{eqn:adapt-two-terms-l1}
    \E[\|\widehat{\bb}^t-\bb^t\|_q^q]
    =\sum_{k\in[d]}\E[\mathds{1}\{k\in{\cK}^t_\alpha\}|\widecheck{b}^{\star}_k-b^t_k|^q]
    +\sum_{k\in[d]}\E[\mathds{1}\{k\notin{\cK}^t_\alpha\}|\widecheck{b}^{t}_k-b^t_k|^q].
\end{equation}

Let $\cI^t\triangleq\{k\in[d]:b_k^t= b_k^\star,\,\ie,\,p_k^t= p_k^\star\}$ be set of entries at which  the $t$-th cluster's distribution $\bp^t$ \emph{aligns} with the central distribution $\bp^\star$.
We next bound the two terms from \eqref{eqn:adapt-two-terms-l1} in 
Lemmas \ref{lem:term-I}  and \ref{lem:term-II}.
These do not need the independence of $\widecheck{\bb}^\star$ and $\widecheck{\bb}^t$, and hence do not require sample splitting despite the division between stages. 
\begin{lemma}\label{lem:term-I} 
For any $t\in[T]$, $\alpha \geq 1$ and $\eta \in(0,1]$, with ${\cK}^t_\alpha$  from \eqref{kt}, we have, for $q=1,2$ 
\begin{equation*}
    \sum_{k\in[d]}\E[\mathds{1}\{k\in{\cK}^t_\alpha\}|\widecheck{b}^\star_k-b^t_k|^q]=O\left(
    \E[\|\widecheck{b}^\star_{\cI_\eta\cap \cI^t }-{b}^\star_{\cI_\eta\cap \cI^t}\|_q^q]+\sum_{k\notin \cI_\eta\cap \cI^t }\left(\frac{\alpha b_k^t}{n}\right)^{q/2}\right).
\end{equation*}
\end{lemma}
\begin{proof}
We first take $q=1$.
For any $k\in [d]$, clearly
\begin{equation}\label{eqn:ohbjrowwe}
\E[\mathds{1}\{k\in{\cK}^t_\alpha\}|\widecheck{b}^\star_k-b^t_k|]\leq \E[|\widecheck{b}^\star_k-b^t_k|].    
\end{equation}
We use this bound for $k\in \cI_\eta \cap \cI^t$.
For $k\notin \cI_\eta \cap \cI^t$, we instead bound
\begin{align*}
    \E[\mathds{1}\{k\in{\cK}^t_\alpha\}|\widecheck{b}^\star_k-b^t_k|]\leq \E[\mathds{1}\{k\in{\cK}^t_\alpha\}|\widecheck{b}^\star_k-\widecheck{b}^t_k|]+\E[\mathds{1}\{k\in{\cK}^t_\alpha\}|\widecheck{b}^t_k-b^t_k|].
\end{align*}
If $k\in{\cK}^t_\alpha$, it holds by definition that $|\widecheck{b}^\star_k-\widecheck{b}^t_k|\leq  \sqrt{\alpha\widecheck{b}^t_k /n}$, thus we further have
\begin{align}
    \E[\mathds{1}\{k\in{\cK}^t_\alpha\}|\widecheck{b}^\star_k-b^t_k|]\leq& \E\left[\mathds{1}\{k\in{\cK}^t_\alpha\}\sqrt{\alpha\widecheck{b}^t_k /n}\right]+\E[\mathds{1}\{k\in{\cK}^t_\alpha\}|\widecheck{b}^t_k-b^t_k|]\nonumber\\
    \leq &\E\left[\sqrt{\alpha\widecheck{b}^t_k /n}\right]+\E[|\widecheck{b}^t_k-b^t_k|].\label{eqn:gvbjoejqqw}
\end{align}
By Jensen's inequality and since $n\widecheck{b}^t_k\sim\mathrm{Binom}(n,b^t_k)$, we have 
\begin{equation}\label{eqn:jgvoegw}
    \E\left[\sqrt{\widecheck{b}^t_k}\right]\leq \sqrt{\E[{\widecheck{b}^t_k}]}=\sqrt{b^t_k}
\end{equation}
and 
\begin{equation}\label{eqn:jgvoegw-222}
    \E[|\widecheck{b}^t_k-b^t_k|]\leq\sqrt{\E[(\widecheck{b}^t_k-b^t_k)^2]}=\sqrt{\frac{b^t_k(1-b^t_k)}{n}}\leq \sqrt{\frac{b^t_k}{n}}.
\end{equation}
Plugging \eqref{eqn:jgvoegw} and \eqref{eqn:jgvoegw-222} into \eqref{eqn:gvbjoejqqw}, we find
\begin{align}
    \E[\mathds{1}\{k\in{\cK}^t_\alpha\}|\widecheck{b}^\star_k-b^t_k|]
    \leq &(\sqrt{\alpha}+1)\sqrt{\frac{b^t_k}{n}}=O\left(\sqrt{\frac{\alpha b^t_k}{n}}\right).\label{eqn:hvieqocwdq}
\end{align}
Summing up \eqref{eqn:ohbjrowwe} over all entries in $\cI_\eta\cap \cI^t$ and summing up \eqref{eqn:hvieqocwdq} over all entries not in $\cI_\eta\cap \cI^t$ leads to the claim for $q=1$ in Lemma \ref{lem:term-I}.  
The case $q=2$ follows by a similar argument.
\end{proof}
\begin{lemma}\label{lem:term-II} For any $t\in[T]$, $\alpha \geq 1$ and $\eta \in(0,1]$, with ${\cK}^t_\alpha$  from \eqref{kt}, we have, for $q=1,2$ 
\begin{align*}
    &\sum_{k\in[d]}\E[\mathds{1}\{k\notin{\cK}^t_\alpha\}
    |\widecheck{b}^{t}_k-b^t_k|^q]\\
    =&O\left(
  \sum_{k\in {\cI_\eta}\cap \cI^t}
  \PP(k\notin {\cK}^t_\alpha)
  \wedge
  \left(\frac{ b_k^t(1-b_k^t)}{n}\right)^{q/2}
  +\sum_{k\notin \cI_\eta\cap \cI^t }
  \left(\frac{b_k^t}{n}\right)^{q/2}
  \right).
\end{align*}
\end{lemma}
\begin{proof}
For $q=1$, note that 
\begin{equation}\label{eqn:vjoweqnfqw}
    \E[\mathds{1}\{k\notin{\cK}^t_\alpha\}|\widecheck{b}^{t}_k-b^t_k|]\leq  \PP(k\notin {\cK}^t_\alpha)
\end{equation}
and 
\begin{equation}\label{eqn:gvjoewnfoq}
    \E[\mathds{1}\{k\notin{\cK}^t_\alpha\}|\widecheck{b}^{t}_k-b^t_k|]\leq\E[|\widecheck{b}^{t}_k-b^t_k|].
\end{equation}
Combining \eqref{eqn:vjoweqnfqw}, \eqref{eqn:gvjoewnfoq} with the first inequality in \eqref{eqn:jgvoegw-222}  for $k\in {\cI_\eta}\cap \cI^t$, and 
using the last inequality in \eqref{eqn:jgvoegw-222}  for $k\notin {\cI_\eta}\cap \cI^t$
leads to the claim with $q=1$. 
We can similarly obtain the bound with $q=2$.
\end{proof}

Combing Lemma \ref{lem:term-I} and \ref{lem:term-II} with \eqref{eqn:adapt-two-terms-l1}, we find the following proposition:
\begin{proposition}\label{prop:main-analysis}
For any $\alpha\geq 1$, and $q=1,2$, it holds that 
\begin{align*}
    &\E[\|\widehat{\bb}^t-\bb^t\|_q^q]
    =O\left(\sum_{k\notin \cI_\eta\cap \cI^t }
    \left(\frac{\alpha b_k^t}{n}\right)^{q/2}\right.\\
    &\left.+\sum_{k\in {\cI_\eta}\cap \cI^t}
    \PP(k\notin {\cK}^t_\alpha)
    \wedge
    \left(\frac{ b_k^t(1-b_k^t)}{n}\right)^{q/2}
    +\E[\|\widecheck{b}^\star_{\cI_\eta\cap \cI^t }-{b}^\star_{\cI_\eta\cap \cI^t}\|_q^q]\right).
\end{align*}
\end{proposition}
Proposition \ref{prop:main-analysis}  does not rely on how $\widecheck{\bb}^\star$ is obtained.
The next part is devoted to proving that when $\widecheck{\bb}^\star$ is obtained via a certain robust estimate,
the  bounds in Proposition \ref{prop:main-analysis}
are small for certain values of $\alpha$ and $\eta$.


\section{Median-Based Method}\label{app:median}
In this section, we provide the proofs for the median-based SHIFT method. 
We first re-state the detailed version of some key results that apply to both the $\ell_2$ and $\ell_1$ errors. 

Below, we use $\sigma_k= \sqrt{b_k^\star(1-b_k^\star)}$ to denote the standard deviation of the $\mathrm{Bernoulli}$ variable with success probability $b_k^\star= p_k^\star+(1-p_k^\star)/2^b$. 
We also recall that $\cB_k$ is defined as the set of clusters with distributions mismatched with the central distribution at the $k$-th entry, \ie, $\cB_k=\{t\in[T]:p^t_k\neq p^\star_k\}$, and $\cI_\eta $ is defined as the $\eta$-well-aligned entries, \ie, $\cI_\eta=\{k\in[d]:|\cB_k|<\eta T\}$.
 
\begin{lemma}[Detailed statement of Lemma \ref{lem:median-entry}]\label{lem:median-entry-combo}
Suppose $\widecheck{\bb}^\star =\med\big(\{\widecheck{\bb}^t\}_{t\in[T]}\big)$. Then for any  $0< \eta \leq  \frac{1}{5}$, $k\in\cI_\eta$, and $q=1,2$, it holds that 
\begin{align*}
    \E[|\widecheck{b}^\star_k-{b}^\star_k|^q]
    =\tilde{O}\left( 
    \left(\frac{|\cB_k|\sigma_k}{T\sqrt{n}}\right)^q
    +\left(\frac{\sigma_k}{\sqrt{Tn}}\right)^q+
    \left(\frac{1}{n}\right)^q\right).
\end{align*}
\end{lemma}
Let us define, for $q=1,2$,
$$
E(q) \triangleq
E(q;n,d,b,T) :=
\frac{d}{(2^bTn)^{q/2}}
    +\frac{d}{n^{q}}.
$$
\begin{proposition}\label{prop:median-global-combo}
Suppose $\widecheck{\bb}^\star =\med\big(\{\widecheck{\bb}^t\}_{t\in[T]}\big)$. Then for any $0< \eta \leq  \frac{1}{5}$ and $q=1,2$, it holds that 
\begin{align*}
    \E[\|\widecheck{b}^\star_{\cI_\eta}-{b}^\star_{\cI_\eta}\|_q^q]
    =\tilde{O}\left(\sum_{k\in\cI_\eta} \left(\frac{|\cB_k|\sigma_k}{T\sqrt{n}}\right)^q
    +E(q)\right).
\end{align*}
\end{proposition}

 We omit the proofs  of Proposition \ref{prop:median-global-combo} and Theorem \ref{thm:transfer-combo} (below),
 as Proposition \ref{prop:median-global-combo} is a direct corollary of Lemma \ref{lem:median-entry-combo} by using $\sum_{k\in[d]}\sigma_k^q=O(d/2^{bq/2})$ for $q=1,2$,
and Theorem \ref{thm:transfer-combo} follows from the same analysis as Theorem \ref{thm:median-collab-final-combo}.

\begin{theorem}[Detailed statement of Theorem \ref{thm:median-collab-final}]\label{thm:median-collab-final-combo}
Suppose  $n\geq  2^{b+6} \ln(n)$ and $\alpha \geq 2(8+\sqrt{8\ln(n)})^2$ with $\alpha =O(\ln(n))$. Then for the median-based SHIFT method, for any $0<\eta \leq \frac{1}{5}$, $q=1,2$, and $t\in[T]$,
\begin{equation*}
 \E\left[\|\widehat{\bp}^t-\bp^t\|_q^q\right]=
 \tilde{O}\left(\sum_{k\notin \cI_\eta\cap \cI^t}\left(\frac{ b_k^t}{n}\right)^{q/2}
 +\sum_{k\in\cI_\eta \cap \cI^t} 
 \left(\frac{|\cB_k|^2b_k^\star}{T^2n}\right)^{q/2}
    +E(q)\right).
\end{equation*}
Furthermore, by setting $\eta = \Theta(1)$ with $\eta\leq \frac{1}{5}$, we have 
\begin{equation*}
 \E\left[\|\widecheck{\bp}^t-\bp^t\|_q^q\right]
 =\tilde{O}\left(
 s^{1-q/2}\left(\frac{\max\{2^b,s\}}{2^bn}\right)^{q/2}
 +E(q).
 \right)
\end{equation*}
\end{theorem}

\begin{theorem}[Detailed statement of Theorem \ref{thm:transfer}]\label{thm:transfer-combo}
Suppose  $n\geq \tilde{n}\geq  2^{b+6} \ln(\tilde{n})$ and $\alpha \geq 2(8+\sqrt{8\ln(\tilde{n})})^2$ with $\alpha =O(\ln(\tilde{n}))$. Then the median-based SHIFT method for predicting the distribution of the new cluster with $\tilde{n}$ users achieves, for $q=1,2$, 
\begin{equation*}
 \E\left[\|\widecheck{\bp}^{T+1}-\bp^{T+1}\|_q^q\right]
 =\tilde{O}\left(
 s^{1-q/2}\left(\frac{\max\{2^b,s\}}{2^b\tilde n}\right)^{q/2}
 +E(q).
 \right)
\end{equation*}
\end{theorem}

\subsection{Proof of Lemma \ref{lem:median-entry-combo}}
We first consider $T \leq 20\ln(n)$. 
In this case, by Bernstein's inequality (Lemma \ref{b}) with $M=1$, we have for any $t\in[T]\backslash\cB_k$ that for any $\delta\ge0$,
\begin{align}\label{eqn:gvhoqnqfwwr-dasd}
    \PP\left(|\widecheck{b}^t_k-{b}^\star_k|> \delta\right)\leq 2e^{-\frac{n}{4}\min\{\delta^2/\sigma_k^2,\delta\}}.
\end{align}
Taking $\delta = \max\{\sigma_k\sqrt{8\ln(n)/n},8\ln(n)/n\}$ in \eqref{eqn:gvhoqnqfwwr-dasd}, we find
\begin{equation}\label{eqn:gvhoqnqfwwr-dasd-dasd}
    \PP\left(|\widecheck{b}^t_k-{b}^\star_k|> \max\left\{\sigma_k\sqrt{\frac{8\ln(n)}{n}},\frac{8\ln(n)}{n}\right\}\right)\leq \frac{2}{n^2}.
\end{equation}
Since $|[T]\backslash \cB_k|>\frac{T}{2}$ for any $k\in\cI_\eta$ with $\eta \leq \frac{1}{5}$, we have, 
since $\widecheck{b}^\star_k =\med\big(\{\widecheck{b}^t_k\}_{t\in[T]}\big)$,
that there are $t_-,t_+\in [T]\backslash \cB_k$ with
$\widecheck{b}^{t'}_k \leq \widecheck{b}^\star_k\leq \widecheck{b}^{t'}_k$.
Hence, $|\widecheck{b}^\star_k-{b}^\star_k|\leq \max\limits_{t\in[T]\backslash\cB_k}|\widecheck{b}^t_k-{b}^\star_k|$. 

Recall that for any random variable $0\leq X\leq 1$ and 
any $\delta\geq0$, $\E[X]\leq \delta + \PP(X\geq\delta)$.
Therefore, by taking the union bound of \eqref{eqn:gvhoqnqfwwr-dasd-dasd} over $k\in [T]\backslash \cB_k$, and by the assumption that  $T \leq 20\ln(n)$, we  have
\begin{align}\label{eqn:gjoqwnqgqw-dsad}
    &\E[|\widecheck{b}^\star_k-{b}^\star_k|]\leq \E[\max\limits_{k\in[T]\backslash\cB_k}|\widecheck{b}^t_k-{b}^\star_k|]\leq \sigma_k\sqrt{\frac{8\ln(n)}{n}}+\frac{8\ln(n)}{n}+ \frac{2T}{n^2}\nonumber\\
    =&O\left(\sigma_k\sqrt{\frac{\ln(n)}{n}}+\frac{\ln(n)}{n}\right)=O\left(\sigma_k\frac{\ln(n)}{\sqrt{Tn}}+\frac{\ln(n)}{n}\right) =\tilde{O}\left(\frac{\sigma_k}{\sqrt{Tn}}+\frac{1}{n}\right).
\end{align}
Similarly, we have 
\begin{align}\label{eqn:gjoqwnqgqw-dsad-dsad}
    \E[(\widecheck{b}^\star_k-{b}^\star_k)^2]\leq \sigma_k^2{\frac{8\ln(n)}{n}}+\frac{64\ln(n)^2}{n^2}+ \frac{2T}{n^2}
   =\tilde{O}\left(\frac{\sigma_k^2}{{Tn}}+\frac{1}{n^2}\right).
\end{align}

For each $k\in[d]$ with $b_k^\star\neq 1$ (recall that $b^\star_k\geq 1/2^b$ by definition), let  $\gamma_k=(1-2b_k^\star(1-b_k^\star))/\sqrt{b_k^\star(1-b_k^\star)}$, 
and let $\tilde{F}_k(x):=\frac{1}{T-|\cB_k|}\sum_{t\in[T]\backslash\cB_k}\mathds{1}(\widecheck{b}^t_k\leq x)$ be the empirical distribution function of
$\{\widecheck{b}^t_k:b_k^t=b_k^\star\}$.
Let $\ep\in(0,1/2)$ and $C_\ep = \sqrt{2\pi}\exp((\Phi^{-1}(1-\ep))^2/2)$.
For $\delta\geq0$, define, recalling $\eta T>|\cB_k|$ for all $k\in\cI_{\eta}$,
$$
G_{k,T,\delta} = \frac{|\cB_k|}{T}+\frac{10^{-8}}{Tn}+\sqrt{\frac{\delta}{T-|\cB_k|}}
$$
where the term $\frac{10^{-8}}{Tn}$ is used to overcome some challenges due to the discreteness of empirical distributions, and can be replaced with other suitably small terms (see the proof of Lemma \ref{lem:median-funda-2}).
Further, define
$$
G'_{k,T,\delta} =
G_{k,T,\delta}+0.4748\frac{\gamma_k}{\sqrt{n}}
.
$$

To prove Lemma \ref{lem:median-entry} for  
$T > 20\ln(n)$, we need the following additional lemmas:
\begin{lemma}\label{lem:median-funda}
For any $\delta\geq 0$ such that
\begin{align}\label{eqn:t-cond}
    G'_{k,T,\delta} \leq \frac{1}{2}-\ep,
\end{align}
it holds with probability at least $1-4e^{-2\delta}$ that  
\begin{align*}
   \tilde{F}_k\left(b_k^\star+C_\ep\frac{{\sigma}_k}{\sqrt{n}}G'_{k,T,\delta}\right)\geq \frac{1}{2}+\frac{|\cB_k|}{T}+\frac{10^{-8}}{Tn}
\end{align*}
and 
\begin{align*}
   \tilde{F}_k\left(b_k^\star-C_\ep\frac{{\sigma}_k}{\sqrt{n}}G'_{k,T,\delta}\right)\leq \frac{1}{2}-\frac{|\cB_k|}{T}-\frac{10^{-8}}{Tn}.
\end{align*}
\end{lemma}
\begin{proof}
The proof essentially follows Lemma 1 of \cite{Yin2018ByzantineRobustDL}. We provide the proof for the sake of being self-contained.

Let $Z_k^t=(\widecheck{b}^t_k-b^t_k)/\sqrt{\mathrm{Var}(\widecheck{b}^t_k)}$ be a standardized version of $\widecheck{b}_k^t$ for each $t\in[T]$ and $k\in[d]$, with $b^\star_k\neq 1$.
Let $\tilde{\Phi}_k(z) = \frac{1}{T-|\cB_k|}\sum_{t\in[T]\backslash\cB_k}\mathds{1}(Z_k^t\leq z)$ be the empirical distribution of $\{Z_k^t:t\in[T]\backslash\cB_k\}$. 
The distribution of $Z_k^t$ is identical $t\in[T]\backslash\cB_k$, and 
we denote by $\Phi_k$ their common cdf.

By definition,  $\E[\tilde{\Phi}_k(z)]=\Phi_k(z)$ for any $z\in\RR$. 
Let  $z_1> 0>  z_2$ be such that $\Phi(z_1)= \frac{1}{2}+G'_{k,T,\delta}$ and $\Phi(z_2)= \frac{1}{2}-G'_{k,T,\delta}$, which exist due to \eqref{eqn:t-cond}. 
Then, by Lemma \ref{lem:berry}, we have
\begin{align}\label{eqn:voqwefnc}
    \Phi_k(z_1)\geq \frac{1}{2}+G_{k,T,\delta}\quad\text{and}\quad\Phi_k(z_2)\leq  \frac{1}{2}-G_{k,T,\delta}.
\end{align}
Further, by the Hoeffding's inequality, we have for any $\delta\geq  0$ and $z\in\RR$, 
\begin{align}\label{eqn:vowendas}
    \left|\tilde{\Phi}_{k}(z)-\Phi_{k}(z)\right| \leq \sqrt{\frac{\delta}{T-|\cB_k|}}
\end{align}
with probability at least $1-2e^{-2\delta}$.
Then, by a union bound of \eqref{eqn:vowendas} for $z=z_1,z_2$, and by \eqref{eqn:voqwefnc}, it holds with  probability at least $1-4e^{-2\delta}$ that
\begin{align}\label{eqn:vopqwnfc}
    \tilde{\Phi}_k(z_1)\geq \frac{1}{2}+\frac{|\cB_k|}{T}+\frac{10^{-8}}{Tn}\quad \text{and}\quad \tilde{\Phi}_k(z_2  )\leq \frac{1}{2}-\frac{|\cB_k|}{T}-\frac{10^{-8}}{Tn}.
\end{align}

Finally, we bound the values of $z_1$ and $z_2$. By the mean value theorem, there exists $\xi \in\left[0, z_{1}\right]$ such that
\begin{align}\label{eqn:vjoqnzv}
    G'_{k,T,\delta}=z_{1} \Phi^{\prime}(\xi)=\frac{z_{1}}{\sqrt{2 \pi}} e^{-\frac{\xi^{2}}{2}} \geq \frac{z_{1}}{\sqrt{2 \pi}} e^{-\frac{z_{1}^{2}}{2}}.
\end{align}
By \eqref{eqn:t-cond} and the definition of $z_1$, we have $z_1\leq \Phi^{-1}(1-\ep)$, and thus, by \eqref{eqn:vjoqnzv}, we have
\begin{align}\label{eqn:vpamcvap}
    z_1\leq \sqrt{2\pi}G'_{k,T,\delta}\exp\left(\frac{1}{2}(\Phi^{-1}(1-\ep))^2\right).
\end{align}
Similarly, we have 
\begin{align}\label{eqn:vpamcvap2}
    z_1\geq - \sqrt{2\pi}G'_{k,T,\delta}\exp\left(\frac{1}{2}(\Phi^{-1}(1-\ep))^2\right).
\end{align}
Since
for all $z$,
$\tilde{\Phi}_k(z)=\tilde{F}_k(\sigma_k z/\sqrt{n }+b_k^\star)$, plugging \eqref{eqn:vpamcvap} and \eqref{eqn:vpamcvap2} into \eqref{eqn:vopqwnfc}, we find the conclusion of this lemma.
\end{proof}

This leads to our next result.
\begin{lemma}\label{lem:median-funda-2}
For any $k\in[d]$ such that condition \eqref{eqn:t-cond} holds,  we have with probability at least $1-4e^{-2\delta}$ that
\begin{align}
\left|\widecheck{b}^t_k-{b}^t_k\right|\leq C_\ep \frac{\sigma_k}{\sqrt{n}}G_{k,T,\delta}^\prime+\frac{0.4748C_\ep}{n}.
\end{align}
\end{lemma}
\begin{proof}
Let $\hat{F}_k$ be the empirical distribution function of $\{\widecheck{b}^t_k:t\in[T]\}$, such that 
for all $x\in\R$, $\hat{F}_k(x):=\frac{1}{T}\sum_{t\in[T]}\mathds{1}(\widecheck{b}^t_k\leq x)$. 
We have
\begin{align}
    |\hat{F}_k(x)-\tilde{F}_k(x)|&=\left|\frac{1}{T}\sum_{t\in[T]}\mathds{1}(\widecheck{b}^t_k\leq x)-\frac{1}{T-|\cB_k|}\sum_{t\in[T]\backslash \cB_k}\mathds{1}(\widecheck{b}^t_k\leq x)\right|\nonumber\\
    &=\left|\frac{1}{T}\sum_{t\in\cB_k}\mathds{1}(\widecheck{b}^t_k\leq x)-\frac{|\cB_k|}{T(T-|\cB_k|)}\sum_{t\in[T]\backslash \cB_k}\mathds{1}(\widecheck{b}^t_k\leq x)\right|\nonumber\\
    &\leq \max\left\{\frac{1}{T}\cdot |\cB_k|,\frac{|\cB_k|}{T(T-|\cB_k|)}\cdot (T-|\cB_k|)\right\}=\frac{|\cB_k|}{T}.\label{eqn:vjwofoasj}
\end{align}
Define $\tilde{F}_k^{-}(x):=\frac{1}{T-|\cB_k|}\sum_{t\in[T]\backslash\cB_k}\mathds{1}(\widecheck{b}^t_k< x)\leq \tilde{F}_k(x)$. 
Then by \eqref{eqn:vjwofoasj} and Lemma \ref{lem:median-funda}, we have, with probability at least $1-4e^{-2\delta}$ that
\begin{equation}\label{eqn:vnoqnrfbqwrq-1}
  \hat{F}_k\left(b_k^\star+C_\ep\frac{{\sigma}_k}{\sqrt{n}}G'_{k,T,\delta}\right)\geq \frac{1}{2}+\frac{10^{-8}}{Tn} \quad\text{and}\quad
  \hat{F}_k^{-}\left(b_k^\star-C_\ep\frac{{\sigma}_k}{\sqrt{n}}G'_{k,T,\delta}\right)\leq \frac{1}{2}-\frac{10^{-8}}{Tn}.
\end{equation}
Let $\widecheck{b}^{(j)}_k$, $\forall\, j\in[T]$ be the $j$-th smallest element in $\{\widecheck{b}^t_k:t\in[T]\}$.
Recalling the definition of the median, 
 if $T$ is odd, then $ \widecheck{b}^\star_k=\widecheck{b}^{((T+1)/2)}_k$. Therefore, $b_k^\star+C_\ep\frac{{\sigma}_k}{\sqrt{n}}G'_{k,T,\delta}<\widecheck{b}^\star_k$ implies  $\hat{F}_k\left(b_k^\star+C_\ep\frac{{\sigma}_k}{\sqrt{n}}G'_{k,T,\delta}\right) \leq \frac{1}{2}-\frac{1}{2T}$ and $b_k^\star-C_\ep\frac{{\sigma}_k}{\sqrt{n}}G'_{k,T,\delta}>\widecheck{b}^\star_k$ implies $\hat{F}_k^{-}\left(b_k^\star-C_\ep\frac{{\sigma}_k}{\sqrt{n}}G'_{k,T,\delta}\right)\geq \frac{1}{2}+\frac{1}{2T}$, leading to a contradiction with \eqref{eqn:vnoqnrfbqwrq-1}.
 
 On the other hand, if $T$ is even, $ \widecheck{b}^\star_k=(\widecheck{b}^{(T/2)}_k+\widecheck{b}^{(T/2+1)}_k)/2$. Therefore, $b_k^\star+C_\ep\frac{{\sigma}_k}{\sqrt{n}}G'_{k,T,\delta}<\widecheck{b}^\star_k$ implies  $\hat{F}_k\left(b_k^\star+C_\ep\frac{{\sigma}_k}{\sqrt{n}}G'_{k,T,\delta}\right) \leq \frac{1}{2}$ and $b_k^\star-C_\ep\frac{{\sigma}_k}{\sqrt{n}}G'_{k,T,\delta}>\widecheck{b}^\star_k$ implies $\hat{F}_k^{-}\left(b_k^\star-C_\ep\frac{{\sigma}_k}{\sqrt{n}}G'_{k,T,\delta}\right)\geq \frac{1}{2}$, which is also contradictory to \eqref{eqn:vnoqnrfbqwrq-1}.
 
To summarize, it holds that
 \[
 |\widecheck{b}^\star_k-{b}^\star_k|\leq C_\ep\frac{{\sigma}_k}{\sqrt{n}}G'_{k,T,\delta}
 \]
 with probability at least $1-4e^{-2\delta}$.
\end{proof}

If $T \leq 20\ln(n)$, Lemma \ref{lem:median-entry-combo} follows directly from \eqref{eqn:gjoqwnqgqw-dsad} and \eqref{eqn:gjoqwnqgqw-dsad-dsad}. 
Now, given Lemma \ref{lem:median-funda} and Lemma \ref{lem:median-funda-2}, we turn to prove Lemma \ref{lem:median-entry-combo} with $T\geq20\ln(n)$. We first check condition \eqref{eqn:t-cond}.  Since $|\cB_k|\leq \eta T$ for any $k\in \cI_\eta$, $\eta \leq \frac{1}{5}$, and $\gamma_k\sigma_k\leq 1$, we have for each $k\in\cI_\eta$ that 
\begin{align*}
    G'_{k,T,\delta}\leq \eta+\frac{10^{-8}}{Tn}+\sqrt{\frac{5\delta}{4T}}+\frac{0.4748}{\sqrt{n}\sigma_k}.
\end{align*}
When $T\geq20\ln(n)$, for any $k\in[d]$ such that $\sigma_k\geq \frac{20}{\sqrt{n}(1-2\eta)}$, taking $\delta = \ln(n)$ above, we have 
\begin{align*}
    G'_{k,T,\delta}\leq \eta+10^{-8}+\frac{1}{4}+0.4748\frac{1-2\eta}{20}\leq \frac{1}{2}-0.035755.
\end{align*}
Therefore, condition \eqref{eqn:t-cond} in Lemma \ref{lem:median-funda-2} is satisfied with $\ep=0.035755$, for which we can check that $C_\ep \leq 13$. 
Thus, for any $\delta \leq \ln(n)$,
\begin{equation}\label{eqn:fhaunvqw}
    \PP\left(|\widecheck{b}^\star_k-{b}^\star_k|\geq 13 \frac{\sigma_k}{\sqrt{n}}G_{k,T,\delta}+\frac{13}{n}\right)\leq 4e^{-2\delta}.
\end{equation}
Therefore, by \eqref{eqn:fhaunvqw}, we have, using that for any random variable $0\leq X\leq 1$ and
any $0\leq r\leq 1$, $\E[X] \leq r + \PP(X\geq r)$,
and that for $\delta = (\ln n)/2$, one has $4e^{-2\delta} = 4/n$, we find
\begin{align}
     &\E[|\widecheck{b}^\star_k-{b}^\star_k|]\leq   13 \frac{\sigma_k}{\sqrt{n}}G_{k,T,(\ln n)/2}+\frac{17}{n}
     =\tilde{O}\left(\frac{{\sigma}_k}{\sqrt{n}}\frac{|\cB_k|}{T}+\frac{\sigma_k}{\sqrt{nT}}+\frac{1}{n}\right)\label{eqn:voqwjfwdas}.
\end{align}
Similarly, by the Cauchy-Schwarz inequality, we also have 
\begin{align}
     \E[(\widecheck{b}^\star_k-{b}^\star_k)^2]=&O\left( \frac{\sigma_k^2}{n}\left(\frac{|\cB_k|^2}{T^2}+\frac{\ln(n)}{T-|\cB_k|}\right)+\frac{1}{n^2}+e^{-2\ln(n)}\right)\nonumber\\
    =&\tilde{O}\left(\frac{{\sigma}_k^2}{{n}}\frac{|\cB_k|^2}{T^2}+\frac{\sigma_k^2}{{nT}}+\frac{1}{n^2}\right)\label{eqn:voqwjfwdas-2}.
\end{align}

On the other hand, for any $k\in[d]\backslash \cB_k$ such that $\sigma_k< \frac{20}{\sqrt{n}(1-2\eta)}$, by Bernstein's inequality  and a union bound, we have 
\begin{align}\label{eqn:gvhoqnqfwwr}
    \PP\left(\max\limits_{k\in[T]\backslash\cB_k}|\widecheck{b}^t_k-{b}^\star_k|> \delta\right)\leq 2(T-|\cB_k|)e^{-\frac{n}{4}\min\{\delta^2/\sigma_k^2,\delta\}}\leq2Te^{-\frac{n}{4}\min\{\frac{ n(1-2\eta)^2\delta^2}{400},\delta\}} .
\end{align}
Since $|[T]\backslash \cB_k|>\frac{T}{2}$, we have as before that $|\widecheck{b}^\star_k-{b}^\star_k|\leq \max\limits_{t\in[T]\backslash\cB_k}|\widecheck{b}^t_k-{b}^\star_k|$. Taking 
\[\delta = 4\max\{\ln(Tn^2), 10\sqrt{\ln(Tn^2)}\}/n\] in \eqref{eqn:gvhoqnqfwwr}, with the same steps as above, we find
\begin{align}\label{eqn:gjoqwnqgqw}
    \E[|\widecheck{b}^\star_k-{b}^\star_k|]\leq& \E[\max\limits_{k\in[T]\backslash\cB_k}|\widecheck{b}^t_k-{b}^\star_k|]\leq \delta + 2Te^{-\frac{n}{4}\min\{\frac{ (1-2\eta)^2n\delta^2}{400},\delta\}}\nonumber\\
    \leq& \frac{4\max\{\ln(Tn^2), 10\sqrt{\ln(Tn^2)}\}+2}{n}=\tilde{O}\left(\frac{1}{n}\right)
\end{align}
and 
\begin{align}\label{eqn:gjoqwnqgqw-2}
    \E[(\widecheck{b}^\star_k-{b}^\star_k)^2]\leq \delta^2 + 2Te^{-\frac{n}{4}\min\{\frac{ (1-2\eta)^2n\delta^2}{400},\delta\}}=\tilde{O}\left(\frac{1}{n^2}\right).
\end{align}
To summarize, combining \eqref{eqn:voqwjfwdas}, \eqref{eqn:voqwjfwdas-2} with   \eqref{eqn:gjoqwnqgqw},\eqref{eqn:gjoqwnqgqw-2}, we complete the proof when $T > 20\ln(n)$.

Furthermore, by using $\sum_{k\in[d]}\sigma_k^q=O(d/2^{bq/2})$ for $q=1,2$, we directly reach Proposition \ref{prop:median-global-combo}.

\subsection{Proof of Theorem \ref{thm:median-collab-final-combo}}
We first consider the case where $T \leq 20\ln(n)$. By definition, $\widehat{b}^t_k$ is either equal to $\widecheck{b}^t_k$ or $\widecheck{b}^\star_k$, and the latter happens only when $k\in\cK^t_\alpha$, \ie, $ |\widecheck{b}^\star_k-\widecheck{b}^t_k|\leq \sqrt{{\alpha\widecheck{b}^t_k}/{n} }$. In this case, we have
\begin{align*}
    |\widehat{b}^t_k-{b}^t_k|&=|\widecheck{b}^\star_k-{b}^t_k|\leq |\widecheck{b}^t_k-{b}^t_k|+|\widecheck{b}^\star_k-\widecheck{b}^t_k|
    \leq |\widecheck{b}^t_k-{b}^t_k|+\sqrt{\frac{\alpha\widecheck{b}^t_k}{n} }.
\end{align*}
Therefore, we have $|\widehat{b}^t_k-{b}^t_k|\leq |\widecheck{b}^t_k-{b}^t_k|+\sqrt{{\alpha\widecheck{b}^t_k}/{n} }$ for all $k\in[d]$.
This leads to 
\begin{align}
    \E[\|\widehat{\bb}^t-{\bb}^t\|_1]\leq& \E[\|\widecheck{\bb}^t-{\bb}^t\|_1]+\sqrt{\frac{\alpha}{n}}\sum_{k\in[d]}\E\left[\sqrt{\widecheck{b}^t_k}\right]\nonumber\\
    \leq& \E[\|\widecheck{\bb}^t-{\bb}^t\|_1]+\sqrt{\frac{\alpha}{n}}\sum_{k\in[d]}\sqrt{\E[\widecheck{b}^t_k]}\label{eqn:voqwdq},
\end{align}
where \eqref{eqn:voqwdq} holds by Jensen's inequality.
By further using the Cauchy–Schwarz inequality, we have 
\begin{equation}\label{eqn:gvjonweqf}
    \E[\|\widecheck{\bb}^t-{\bb}^t\|_1]\leq \sqrt{d\,\E[\|\widecheck{\bb}^t-{\bb}^t\|_2^2]}=O\left(\frac{d}{\sqrt{2^bn}}\right)
\end{equation}
and 
\begin{equation}\label{eqn:gvjonweqffsaf}
    \sum_{k\in[d]}\sqrt{\E[\widecheck{b}^t_k]}=\sum_{k\in[d]}\sqrt{{b}^t_k}\leq \sqrt{d\,\sum_{k\in[d]}b^t_k}=O\left(\frac{d}{\sqrt{2^b}}\right).
\end{equation}
Plugging \eqref{eqn:gvjonweqf} and \eqref{eqn:gvjonweqffsaf} into \eqref{eqn:voqwdq}, we find 
\begin{equation*}
     \E[\|\widecheck{\bb}^t-{\bb}^t\|_1]=\tilde{O}\left(\frac{d}{\sqrt{2^bn}}\right)=\tilde{O}\left(\frac{d}{\sqrt{2^bTn}}\right).
\end{equation*}
We can similarly prove
\begin{equation*}
     \E[\|\widecheck{\bb}^t-{\bb}^t\|_2^2]=\tilde{O}\left(\frac{d}{{2^bn}}\right)=\tilde{O}\left(\frac{d}{{2^bTn}}\right).
\end{equation*}
Next we prove the case where $T\geq 20\ln(n)=\Omega(\ln(n))$. We first consider the estimation errors over  $k\in\cI_\eta \cap \cI^t$ such that $\sigma_k\geq \frac{20}{\sqrt{n}(1-2\eta)}$.  
Let $\cE_k^t:=\{\widecheck{b}^{t}_k\geq  \frac{1}{2}{b}^{t}_{k}\text{ and }|\widecheck{b}^\star_k-{b}^\star_k|\leq8 \sqrt{{ {b}^\star_k}/{n}}\}$. 
If $n\geq 2^{b+6}\ln(n)$ and $0<\eta\leq 1/5$, then since ${b}^\star_k\geq \frac{1}{2^b}$ for any $k\in[d]$, we have 
\begin{align*}
    &13\frac{\sigma_k}{\sqrt{n}}
G_{k,T,\ln n}
+\frac{13}{n}=13\frac{\sigma_k}{\sqrt{n}}\left(\frac{|\cB_k|}{T}+\frac{10^{-8}}{Tn}+\sqrt{\frac{\ln(n)}{T-|\cB_k|}}\right)+\frac{13}{n}\\
\leq& 13\frac{\sigma_k}{\sqrt{n}}\left(\frac{|\cB_k|}{T}+\frac{10^{-8}}{Tn}+\sqrt{\frac{5\ln(n)}{4T}}\right)+\frac{13}{n}\leq 13\frac{{\sigma_k}}{\sqrt{n}}\left(\frac{1}{5}+10^{-8}+\frac{1}{4}\right)+\frac{13}{\sqrt{n 2^{b+6}\ln(n)}}\\
\leq &13\frac{{\sigma_k}}{\sqrt{n}}\left(\frac{1}{5}+10^{-8}+\frac{1}{4}\right)+\frac{13\sqrt{b^\star_k}}{\sqrt{n 64\ln(n)}}\leq 8\sqrt{\frac{{b}^\star_k}{n}}.
\end{align*}
Hence, by \eqref{eqn:fhaunvqw}, it holds that
\begin{equation}\label{eqn:fqwoindfqw}
    \PP\left(|\widecheck{b}^\star_k-{b}^\star_k|\geq8\sqrt{\frac{ {b}^\star_k}{n}}\right)\leq \frac{4}{n^2}.
\end{equation}
By Bernstein's inequality and as ${b}^\star_k\geq \frac{1}{2^b}$, we have 
\begin{align}
    \PP\left(|\widecheck{b}^t_k-{b}^t_k|>\frac{b^t_k}{2}\right)&\leq 2e^{-\frac{n}{4}\min\{\frac{b^t_k}{4(1-b^t_k)},\frac{b^t_k}{2}\}}\leq 2e^{-\frac{nb^t_k}{16}}
    \leq 2e^{-\frac{n}{16\cdot 2^b}}\leq \frac{2}{n^2},\label{eqn:gjoqwnfq}
\end{align}
where the last inequality holds because $n\geq 2^{b+6}\ln(n)$.
Combining \eqref{eqn:gjoqwnfq} with \eqref{eqn:fqwoindfqw}, we find $\PP((\cE_k^t)^c)\leq\frac{6}{n^2}$. By definition, $k\notin {\cK}_\alpha^t$ implies $ |\widecheck{b}^\star_k-\widecheck{b}^t_k|>\sqrt{{\alpha \widecheck{b}^t_k}/{n}}$. 
On the event $\cE_k^t$, this further implies $ |\widecheck{b}^\star_k-\widecheck{b}^t_k|>\sqrt{{\alpha {b}^t_k}/{2n}}$. Combined with \eqref{eqn:fqwoindfqw} and that ${b}^\star_k=b^t_k$ for any $k\in \cI^t$, we have on the event $\cE_k^t$
\begin{align}\label{eqn:vjqonvqofa}
     \left|\widecheck{b}^{t}_k-{b}^t_k\right|=&\left|\widecheck{b}^{t}_k-{b}^\star_k\right|\geq \left|\widecheck{b}^{t}_k-\widecheck{b}^\star_k\right|-\left|\widecheck{b}^\star_k-{b}^\star_k\right|
     >
     \sqrt{\frac{{b}^t_k}{n}}\left(\sqrt{\frac{\alpha}{2}}-8\right).
\end{align}
Let $\zeta \triangleq\sqrt{{\alpha}/{2}}-8\geq \sqrt{8\ln(n)}$ and $\cF_k^t:=\left\{\left|\widehat{b}^{t}_k-{b}^t_k\right|\geq\zeta \sqrt{{b^t_k}/{n}}\right\}$.
 By Bernstein's inequality, and using  $n\geq 2^{b+6}\ln(n)$, we have 
\begin{align}
    \PP(\cF_k^t)\leq 2e^{-\frac{n}{4}\min\{\frac{\zeta^2}{n(1-b^t_k)},\zeta\sqrt{\frac{b^t_k}{n}}\}}\leq 2e^{-\min\{\frac{\zeta^2}{4},\frac{\zeta}{4}\sqrt{\frac{n}{2^b}}\}}\label{eqn:vjoqwndaa2}\leq \frac{2}{n^2 }.
\end{align}
Combining \eqref{eqn:vjqonvqofa} with \eqref{eqn:vjoqwndaa2}, we find that for any $k\in\cI_\eta\cap\cI^t$ with $\sigma_k\geq \frac{20}{\sqrt{n}(1-2\eta)}$, it holds that
\begin{align*}
    &\PP(k\notin {\cK}^t_\alpha)\leq\PP((\cE_k^t)^c)+\PP(\cE_k^t\cap \{k\notin {\cK}^t_\alpha\})\leq \PP((\cE_k^t)^c)+\PP(\cE_k\cap\cF_k^t )\\
    \leq&  \PP((\cE_k^t)^c)+\PP(\cF_k^t )\leq \frac{8}{n^2}.
\end{align*}
On the other hand for any $k\in\cI_\eta\cap\cI^t$ with $\sigma_k< \frac{20}{\sqrt{n}(1-2\eta)}$, we have 
\begin{equation*}
     \sqrt{\frac{b^t_k(1-b^t_k)}{n}}=\sqrt{\frac{b^\star_k(1-b^\star_k)}{n}}=\frac{\sigma_k}{\sqrt{n}}=O\left(\frac{1}{n}\right).
\end{equation*}
Therefore, we have for all $k\in\cI_\eta\cap\cI^t$, and $q=1,2$
\begin{equation}\label{eqn:fjfoqmqffq}
    \min\left\{\PP(k\notin {\cK}^t_\alpha),\left(\frac{b^t_k(1-b^t_k)}{n}\right)^{q/2}
    \right\}=O\left(\frac{1}{n^q}\right).
\end{equation}
Since $\alpha=O(\ln(n)) $, by \eqref{eqn:fjfoqmqffq} and Proposition \ref{prop:main-analysis}, we obtain
\begin{align}\label{eqn:gjoegnqogq}
     \E[\|\widehat{\bb}^t-\bb^t\|_1]=\tilde{O}\left(\sum_{k\notin \cI_\eta\cap \cI^t }\sqrt{\frac{ b_k^t}{n}}+ \E[\|\widecheck{b}^\star_{\cI_\eta\cap \cI^t }-{b}^\star_{\cI_\eta\cap \cI^t}\|_1]+\frac{d}{n}\right).
\end{align}
Combining \eqref{eqn:gjoegnqogq} with Proposition \ref{prop:median-global-combo} and using that $\sigma_k\leq \sqrt{b_k^\star}=\sqrt{b_k^t}$ for any $k\in\cI^t$, we have 
\begin{align}\label{eqn:gjoegnqogq-fruther}
     &\E[\|\widehat{\bb}^t-\bb^t\|_1]=\tilde{O}\left(\sum_{k\notin \cI_\eta\cap \cI^t }\sqrt{\frac{ b_k^t}{n}}+\sum_{k\in\cI_\eta\cap\cI^t} \frac{|\cB_k|}{T}\sqrt{\frac{b^t_k}{n}}+E(1)\right).
\end{align}
Since $|(\cI^t)^c|=\|\bp^t-\bp^\star\|_0\leq s$, by the Cauchy-Schwarz inequality, we have
\begin{align}
    &\sum_{k\notin \cI_\eta\cap \cI^t }\sqrt{\frac{ b_k^t}{n}}\leq \sum_{k\notin \cI_\eta }\sqrt{\frac{ b_k^t}{n}}+\sum_{k\notin \cI^t }\sqrt{\frac{ b_k^t}{n}}\leq \sum_{k\notin \cI_\eta }\sqrt{\frac{ b_k^t}{n}}+\sqrt{\frac{s\sum_{k\notin \cI^t } b^t_k}{n}}\nonumber\\
   \leq & \sum_{k\notin \cI_\eta }\sqrt{\frac{ b_k^t}{n}}+\sqrt{\frac{s\sum_{k\notin \cI^t }((2^b-1)p^t_k+1)}{2^bn}}
   \leq \sum_{k\notin \cI_\eta }\sqrt{\frac{ b_k^t}{n}}+\sqrt{\frac{s(2^b-1+s)}{2^bn}}.\label{eqn:voendasdfqqw}
\end{align}
Plugging \eqref{eqn:voendasdfqqw} into \eqref{eqn:gjoegnqogq}, we further obtain 
\begin{align}\label{eqn:gjoegnqogq-gdjgfasj}
     \E[\|\widehat{\bb}^t-\bb^t\|_1]=\tilde{O}\left(\sum_{k\notin \cI_\eta}\sqrt{\frac{ b_k^t}{n}}+\sum_{k\in\cI_\eta} \frac{|\cB_k|}{T}\sqrt{\frac{b^t_k}{n}}+\sqrt{\frac{s\max\{2^b,s\}}{2^bn}}+E(1)\right).
\end{align}
Similarly, we can reach the following $\ell_2$ counterpart:
\begin{align}\label{eqn:gjoegnqogq-further-l2}
    \E[\|\widehat{\bb}^t-\bb^t\|_2^2]=\tilde{O}\left(\sum_{k\notin \cI_\eta}{\frac{ b_k^t}{n}}+\sum_{k\in\cI_\eta }\frac{|\cB_k|^2}{T^2}\frac{b_k^t}{n}+\frac{\max\{2^b,s\}}{2^bn}+E(2)\right).
\end{align}

Note that $\sum_{k\in[d]}|\cB_k|/T\leq s$ and for any set $\cI$ with $|\cI|=\lceil \frac{s}{\eta}\rceil$,
\[
\sum_{k\in\cI }\sqrt{\frac{ b_k^t}{n}}\leq \sqrt{\frac{|\cI|\sum_{k\in\cI}((2^b-1)p^t_k+1)}{2^bn}}=O\left(\sqrt{\frac{s/\eta\max\{2^b,s/\eta\}}{2^bn}}\right).
\]
Now, recalling the definition of $\cI_\eta$, 
we apply Lemma \ref{lem:f-g} in \eqref{eqn:gjoegnqogq-gdjgfasj} with $(r_k,x_k)=(\sqrt{b^t_k/n},|\cB_k|/T)$ for all $k\in[d]$, to find 
\begin{align*}
     \E[\|\widehat{\bb}^t-\bb^t\|_1]=\tilde{O}\left(\sqrt{\frac{s/\eta\max\{2^b,s/\eta\}}{2^bn}}+E(1)\right).
\end{align*}
Therefore, for any $\eta = \Theta(1)$ with $\eta\leq \frac{1}{5}$, we finally have 
\begin{align*}
     \E[\|\widehat{\bb}^t-\bb^t\|_1]=\tilde{O}\left(\sqrt{\frac{s\max\{2^b,s\}}{2^bn}}+E(1)\right).
\end{align*}
Similarly, by combining \eqref{eqn:gjoegnqogq-further-l2} with Lemma \ref{lem:f-g}, we have for any $\eta = \Theta(1)$ with $\eta\leq \frac{1}{5}$,
\begin{align*}
    \E[\|\widehat{\bb}^t-\bb^t\|_2^2]=\tilde{O}\left(\frac{\max\{2^b,s\}}{2^bn}+E(2)\right).
\end{align*}
The result directly follows Proposition \ref{prop:debias}.

\begin{lemma}\label{lem:f-g}
Given $\eta\in(0,1]$, $r_k\geq 0$ for all $k\in[d]$,
and for $q=1,2$,
consider the functions $f_q: \{x\in\R^d: 0\leq x_k\leq 1,\,\forall\, k\in[d] \text{ and }\sum_{k\in[d]}x_k\leq s\}\to\R$,  $f_q(x_1,\dots,x_d):=\sum_{k\in[d]}r_k^q(\mathds{1}\{x_k\geq \eta \}+x_k^q\mathds{1}\{x_k<\eta \})$.
Then it holds that 
\begin{equation}\label{eqn:gvjoweqfgedgsd}
    \max_{x_1,\dots,x_d}f_q(x_1\dots,x_d)\leq \sum_{k=1}^{\lceil s/\eta \rceil}r_{(k)}^q,
\end{equation}
where $r_{(1)}\geq \cdots \geq r_{(d)}$ is the non-decreasing rearrangement of $\{r_1,\dots,r_d\}$.
\end{lemma}
\begin{proof}
We only prove the result for $f_1$, and the result for function $f_2$ follows similarly.
Note that $r_k(\mathds{1}\{x_k\geq \eta \}+x_k\mathds{1}\{r_k\geq \eta \})$ is increasing with respect to $r_k$ and $x_k$. To consider the maximum of the sum in $f$, 
by the rearrangement inequality, without loss of generality, we can assume  $r_1\geq r_2\geq \cdots\geq r_d\geq0$ and $1\geq x_1\geq x_2\geq \cdots\geq x_d\geq0$. In this case, we claim that the maximum is attained at $x_1=\cdots=x_{\lfloor s/\eta \rfloor}=\eta$, $x_{\lfloor s/\eta \rfloor+1}=s-\eta\lfloor s/\eta \rfloor $, and $x_k=0$ for all $k>\lfloor s/\eta \rfloor+1$.
Further, the maximum is $\sum_{k=1}^{\lfloor s/\eta \rfloor}r_k+r_{\lfloor s/\eta \rfloor+1}(s-\eta\lfloor s/\eta \rfloor)^2$, which is upper bounded by the right-hand side of \eqref{eqn:gvjoweqfgedgsd}. We now use the exchange argument to prove the claim.
\begin{itemize}
    \item[\quad S. 1] If there is some $k$ such that $x_k>\eta \geq x_{k+1}$, then defining $x'$ by letting $(x_k^\prime ,x_{k+1}^\prime )=(\eta, x_k+x_{k+1}-\eta)$ while for other $j$, $x'_j=x_j$, increases the value of $f$. Therefore, the maximum is attained by $x$ such that  for some $j$, $x_1 = \dots = x_j = \eta > x_{j+1} \geq  \cdots \geq x_d$.
    \item[S. 2] If  there is some $k$ such that $\eta >x_k\geq   x_{k+1}>0$, then defining $x'$ by letting  $(x_k^\prime ,x_{k+1}^\prime )=(\min\{\eta,x_k+x_{k+1}\}, \max\{0,x_k+x_{k+1}-\eta\})$  while for other $j$, $x'_j=x_j$, increases the value of $f$. 
    Therefore, combined with Step 1, the maximum is attained by $x$ such that  for some $j$,  $x_1 = \dots = x_j = \eta > x_{j+1} \geq  0$ and $x_k=0$ for all $k>j+1$. 
    Thus most one element lies in $(0,\eta)$.
\end{itemize}
Combining S. 1 and S. 2 above, we complete the proof of the claim, which further leads to \eqref{eqn:gvjoweqfgedgsd}.
\end{proof}

\section{Trimmed-Mean-Based Method}\label{app:trimmed-mean}
In this section, we study the trimmed-mean-based estimator. 
Fix $\omega \in (0,1/2)$.
Specifically, 
for each $k\in[d]$, let $\cU_k$ be the subset of $\{[\widecheck{\bp}^t]_{t\in [T]}\}$ obtained by removing the largest and smallest $\omega T$ elements\footnote{To be precise, one can either trim $\lceil \omega T\rceil$ or $\lfloor \omega T\rfloor$ elements. From now on, we write $\omega T$ for conciseness without further notice.}. 
Then, the trimmed-mean-based method can be expressed as 
\begin{equation}\label{eqn:dsajifa}
    \widecheck{b}^\star_k= \frac{1}{|\cU_k|}\sum_{t\in\cU_k}\widecheck{b}^t_k.
\end{equation}
We also write $\widecheck{\bb}^\star =\trmean\big(\{\widecheck{\bb}^t\}_{t\in[T]},\omega\big)$.
For any chosen trimming proportion $0\leq \eta \leq \omega\leq \frac{1}{5}$, we control the estimation error of each $\eta$-well aligned entry. 
Intuitively, this is small because there are at most a fraction of $\eta$ elements from heterogeneous distributions. These are trimmed if they behave as outliers, and otherwise kept in $\cU_k$. 
The error control for a single entry $k\in\cI_{\eta}$ is in Lemma \ref{lem:fund-trim}.
\begin{lemma}\label{lem:fund-trim}
Suppose $\widecheck{\bb}^\star =\trmean\big(\{\widecheck{\bb}^t\}_{t\in[T]},\omega\big) $ such that $0\leq \omega\leq \frac{1}{5}$. Then for each $k\in\cI_\eta$ with $0<\eta \leq \omega$ and any $q=1,2$, it holds that
\begin{align}\label{eqn:gvjoqmfqw}
    \E[|\widecheck{b}^\star_k-{b}^\star_k|^q]=\tilde{O}\left(\left(\omega^2\frac{b_k^\star}{n}\right)^{q/2}+\left(\frac{b_k^\star}{Tn}\right)^{q/2}+\frac{1}{(Tn)^q}+\left(\frac{\omega}{n}\right)^q\right).
\end{align}
\end{lemma}
\begin{proof} To prove Lemma \ref{lem:fund-trim}, we need the following lemma.
\begin{lemma}\label{lem:fund-trim-prob}
For each $k\in\cI_\eta$ with $0<\eta \leq \omega \leq \frac{1}{5}$, and any $\ep_k,\delta_k\geq 0$,  it holds with probability at least $1-2e^{-\frac{(T-|\cB_k|)n}{4}\min\{\frac{\ep_k^2}{\sigma_k^2},\ep_k\}}-2(T-|\cB_k|)e^{-\frac{n}{4}\min\{\frac{\delta_k^2}{\sigma_k^2},\delta_k\}}$ that
\begin{align*}
    |\widecheck{b}^\star_k-{b}^\star_k|\leq \frac{\ep_k+3 \omega \delta_k}{1-2\omega}.
\end{align*}
\end{lemma}
\begin{proof}[Proof of Lemma \ref{lem:fund-trim-prob}]
By Bernstein's inequality and  the union bound, we have for any $\ep_k,\,\delta_k>0$ that
\begin{align*}
    \PP\left(\left|\frac{1}{T-|\cB_k|}\sum_{t\in[T]\backslash \cB_k}\widecheck{b}^t_k-{b}^\star_k\right|> \ep_k\right)\leq 2e^{-\frac{(T-|\cB_k|)n}{4}\min\{\frac{\ep_k^2}{\sigma_k^2},\ep_k\}}
\end{align*}and
\begin{align*}
    \PP\left(\max\limits_{t\in[T]\backslash\cB_k}|\widecheck{b}^t_k-{b}^\star_k|> \delta_k\right)\leq 2(T-|\cB_k|)e^{-\frac{n}{4}\min\{\frac{\delta_k^2}{\sigma_k^2},\delta_k\}}.
\end{align*}
 By the definition of $\widecheck{b}^\star_k$, we have
\begin{align*}
    &|\widecheck{b}^\star_k-{b}^\star_k|=\frac{1}{T(1-2 \omega)}\left|\sum_{t\in\cU_k}\widecheck{b}^t_k-{b}^\star_k\right|\\
    =&\frac{1}{T(1-2 \omega)}\left|\sum_{t\in[T]\backslash\cB_k}(\widecheck{b}^t_k-{b}^\star_k)-\sum_{t\in[T]\backslash(\cB_k\cup \cU_k)}(\widecheck{b}^t_k-{b}^\star_k)+\sum_{t\in\cB_k\cap \cU_k}(\widecheck{b}^t_k-{b}^\star_k)\right|\\
    \leq& \frac{1}{T(1-2 \omega )}\left(\left|\sum_{t\in[T]\backslash\cB_k}\widecheck{b}^t_k-{b}^\star_k\right|+\left|\sum_{i\notin\cB_k\cup \cU_k}\widecheck{b}^t_k-{b}^\star_k\right|+\left|\sum_{t\in\cB_k\cap \cU_k}\widecheck{b}^t_k-{b}^\star_k\right|\right).
\end{align*}
It is clear that 
\begin{align*}
    \left|\sum_{t\in[T]\backslash(\cB_k\cup \cU_k)}(\widecheck{b}^t_k-{b}^\star_k)\right|\leq |[T]\backslash\cU_k|\max_{t\in[T]\backslash\cB_k}|\widecheck{b}^t_k-{b}^\star_k|=2\omega T\max_{t\in[T]\backslash\cB_k}|\widecheck{b}^t_k-{b}^\star_k|.
\end{align*}
Then we claim that $\left|\sum_{t\in\cB_k\cap \cU_k}\widecheck{b}^t_k-{b}^\star_k\right|\leq |\cB_k|\max_{t\in[T]\backslash\cB_k}|\widecheck{b}^t_k-{b}^\star_k|$. Let $\cQ_{k,\mathrm{l}}$ and $\cQ_{k,\mathrm{r}}$ be the indices of the trimmed elements on the left side and right side, respectively, 
\ie, the smallest and largest $\omega T$ elements among $\{\widecheck{b}^t_k\}_{t\in[T]}$.
If $\cB_k\cap \cU_k\neq \emptyset$, then $|\cU_k\backslash\cB_k|<T(1-2\omega)$.
Furthermore, we have $|\cQ_{k,\mathrm{l}}\cup(\cU_k\backslash\cB_k)|=|\cQ_{k,\mathrm{r}}\cup(\cU_k\backslash\cB_k)|=\omega T+|\cU_k\backslash\cB_k|<T(1- \omega)\leq |[T]\backslash \cB_k|$, 
which leads to $([T]\backslash \cB_k)\cap \cQ_{k,\mathrm{l}}\neq \emptyset$ and $(T\backslash \cB_k)\cap \cQ_{k,\mathrm{r}}\neq \emptyset$.
In conclusion, we have $\max_{t\in\cU_k}|\widecheck{b}^t_k-{b}^\star_k|\leq \max_{t\in [T]\backslash \cB_k}|\widecheck{b}^t_k-{b}^\star_k|$, which completes the proof  of the claim. Therefore, we have 
\begin{align*}
     &|\widecheck{b}^\star_k-{b}^\star_k|
     \leq \frac{1}{T(1-2 \omega)}\left(\left|\sum_{t\in[T]\backslash\cB_k}|\widecheck{b}^t_k-{b}^\star_k|\right|+(2\omega T+|\cB_k|)\max_{t\in[T]\backslash\cB_k}|\widecheck{b}^t_k-{b}^\star_k|\right)
     \leq \frac{\ep_k+3\omega \delta_k}{1-2\omega}
\end{align*}
with probability at least $1-2e^{-\frac{(T-|\cB_k|)n}{4}\min\{\frac{\ep_k^2}{\sigma_k^2},\ep_k\}}-2(T-|\cB_k|)e^{-\frac{n}{4}\min\{\frac{\delta_k^2}{\sigma_k^2},\delta_k\}}$.
\end{proof}
Given Lemma \ref{lem:fund-trim-prob}, by setting
\begin{equation*}
\ep_k= \max\left\{\frac{4\sigma_k\sqrt{\ln(T^2n^2)}}{\sqrt{(T-|\cB_k|)n}},\frac{8\ln(T^2n^2)}{(T-|\cB_k|)n}\right\}=\tilde{O}\left(\frac{\sigma_k}{\sqrt{Tn}}+\frac{1}{Tn}\right)
\end{equation*} 
and
\begin{equation*}
\delta_k=\max\left\{\frac{4\sigma_k\sqrt{\ln(T^2(T-|\cB_k|)n^2)}}{\sqrt{n}}, \frac{4\ln(T^2(T-|\cB_k|)n^2)}{n}\right\}=\tilde{O}\left(\frac{\sigma_k}{\sqrt{n}}+\frac{1}{n}\right), 
\end{equation*}
using that $1/(1-2\omega)\leq \frac{5}{3}$, and recalling $\sigma_k\leq \sqrt{b_k^\star}$,
we have with probability at least $1-\frac{4}{T^2n^2}$ that 
\begin{align}
    &|\widecheck{b}^\star_k-{b}^\star_k|\leq \frac{\ep_k+3 \omega\delta_k}{1-2\omega }\nonumber\\
    \leq& \frac{5\omega}{3} \max\left\{\frac{4\sqrt{b_k^\star\ln(T^3n^2)}}{\sqrt{n}}, \frac{4\ln(T^3n^2)}{n}\right\}+ \frac{5}{3}\max\left\{\frac{4\sqrt{b_k^\star\ln(T^2n^2)}}{\sqrt{(T-|\cB_k|)n}},\frac{4\ln(T^2n^2)}{(T-|\cB_k|)n}\right\}\label{eqn:jfownfqwfrwgwq}\\
    =&\tilde{O}\left(\omega\sqrt{\frac{b_k^\star}{{n}}}+\frac{\omega}{n}+\frac{\sigma_k}{\sqrt{Tn}}+\frac{1}{Tn}\right),\nonumber
\end{align}
which implies 
\begin{align*}
    \E[|\widecheck{b}^\star_k-{b}^\star_k|]=&\tilde{O}\left(\omega\sqrt{\frac{b_k^\star}{{n}}}+\frac{\omega}{n}+\frac{\sigma_k}{\sqrt{Tn}}+\frac{1}{Tn}+\frac{1}{T^2n^2}\right)\\
    =&\tilde{O}\left(\omega\sqrt{\frac{b_k^\star}{{n}}}+\sqrt{\frac{b_k^\star}{Tn}}+\frac{1}{Tn}+\frac{\omega}{n}\right).
\end{align*}
Similarly, we can obtain 
\begin{align*}
    \E[(\widecheck{b}^\star_k-{b}^\star_k)^2]=&\tilde{O}\left(\frac{\omega^2b_k^\star}{{n}}+\frac{\omega^2}{n^2}+\frac{\sigma_k^2}{{Tn}}+\frac{1}{T^2n^2}+\frac{1}{T^2n^2}\right)\\
    =&\tilde{O}\left(\omega^2{\frac{b_k^\star}{{n}}}+{\frac{b_k^\star}{Tn}}+\frac{1}{T^2n^2}+\frac{\omega^2}{n^2}\right).
\end{align*}
\end{proof}

Given these results, we readily establish the following bound on the total error over all $\eta$-well-aligned entries. 
\begin{proposition}\label{prop:trim-global}
Suppose $\widecheck{\bb}^\star =\trmean\big(\{\widecheck{\bb}^t\}_{t\in[T]},\omega\big) $ such that $0\leq \omega\leq 1/5$. Then for each $k\in\cI_\eta$ with $0<\eta \leq \omega$ and any $q=1,2$, it holds that
\begin{align*}
    \E[\|\widecheck{b}^\star_{\cI_\eta}-{b}^\star_{\cI_\eta}\|_q^q]=\tilde{O}\left(d\left(\frac{\omega^2}{{2^b n}}\right)^{q/2}+{\frac{d}{(2^b Tn)^{q/2}}} +\frac{d}{(Tn)^q}+d\left(\frac{\omega}{n}\right)^q\right).
\end{align*}
\end{proposition}

By setting $\alpha = \Theta(\ln(Tn))$, we find the following result.
\begin{theorem}\label{thm:trim-mean-collab-final}
Suppose  $n\geq 2^{b+5}\ln(Tn)$ and $\alpha \geq 2(8+\sqrt{8\ln(Tn)})^2$ with $\alpha =O(\ln(Tn))$. Then for the trimmed-mean-based SHIFT method, for any $0<\omega \leq \frac{1}{5}$, $t\in[T]$ and $q=1,2$,
\begin{equation*}
 \E\left[\|\widehat{\bp}^t-\bp^t\|_q^q\right]=\tilde{O}\left(\left(\frac{s}{\omega}\right)^{1-q/2}\left(\frac{\max\{2^b,s/\omega\}}{2^bn}\right)^{q/2}+ d\left(\frac{\omega^2}{{2^b n}}\right)^{q/2}+{\frac{d}{(2^b Tn)^{q/2}}}\right).
\end{equation*}
\end{theorem}
\begin{proof}
To apply Proposition \ref{prop:main-analysis}, we need to bound $\sum_{k\in {\cI_\eta}\cap \cI^t}\min\{\PP(k\notin {\cK}^t_\alpha),\sqrt{{ b_k^t( 1-b_k^t)}/{n}}\}$ and $\sum_{k\in {\cI_\eta}\cap \cI^t}\min\{\PP(k\notin {\cK}^t_\alpha),{{ b_k^t( 1-b_k^t)}/{n}}\}$.

Let $\cE_k^t:=\{\widecheck{b}^{t}_k\geq  \frac{1}{2}{b}^{t}_{k}\text{ and }|\widecheck{b}^\star_k-{b}^\star_k|\leq8 \sqrt{{ {b}^\star_k}
\ln(T^3n^2)/{n}}\}$. For each  entry $k\in\cI_\eta\cap\cI^t$, since $n\geq 2^b\ln(T^3n^2)$ and $ b_k^\star \leq  \frac{1}{2^b}$, we have $\frac{1}{n}\leq\sqrt{\frac{b_k^\star}{n\ln(T^3n^2)}}$.
By \eqref{eqn:jfownfqwfrwgwq}, we have with probability at least $1-\frac{4}{T^2n^2}$ that
\begin{align}
    &|\widecheck{b}^t_k-b^t_k| 
    \leq \frac{5\omega}{3} \max\left\{\frac{4\sqrt{b_k^\star\ln(T^3n^2)}}{\sqrt{n}}, \frac{4\ln(T^3n^2)}{n}\right\}+ \frac{5}{3}\max\left\{\frac{4\sqrt{b_k^\star\ln(T^2n^2)}}{\sqrt{(T-|\cB_k|)n}},\frac{4\ln(T^2n^2)}{(T-|\cB_k|)n}\right\}\nonumber\\
\leq&  \frac{4}{3}\sqrt{\frac{b_k^\star\ln(T^3n^2)}{n}}+\frac{20}{3}\sqrt{\frac{b_k^\star\ln(T^2n^2)}{(T-|\cB_k|)n}}\leq 8\sqrt{\frac{b_k^\star\ln(T^3n^2)}{n}}.\label{eqn:fhiwnfqowfrdqw}
\end{align}
By Bernstein's inequality and as ${b}^\star_k\geq \frac{1}{2^b}$, we have 
\begin{align}
    \PP\left(|\widecheck{b}^t_k-{b}^t_k|>\frac{b^t_k}{2}\right)&\leq 2e^{-\frac{n}{4}\min\{\frac{b^t_k}{4(1-b^t_k)},\frac{b^t_k}{2}\}}\leq 2e^{-\frac{nb^t_k}{16}}
    \leq 2e^{-\frac{n}{16\cdot 2^b}}\leq \frac{2}{T^2n^2},\label{eqn:gjoqwnfq-trineqwr}
\end{align}
where the last inequality is because $n\geq 2^{b+5}\ln(Tn)$.
Combining \eqref{eqn:fhiwnfqowfrdqw} with \eqref{eqn:gjoqwnfq-trineqwr}, we find $\PP((\cE_k^t)^c)\leq\frac{6}{T^2n^2}$.
Now following the argument from \eqref{eqn:vjqonvqofa}-\eqref{eqn:fjfoqmqffq}, we can obtain that for all $k\in\cI_\eta\cap\cI^t$,
\begin{equation*}
  \PP(k\notin {\cK}^t_\alpha)=O\left(\frac{1}{T^2n^2}\right).
\end{equation*}

Since $\alpha=O(\ln(Tn)) $, by applying \eqref{eqn:fjfoqmqffq} to Proposition \ref{prop:main-analysis} with $\eta=\omega$ and using Proposition \ref{prop:trim-global} with $n=\Omega(2^b)$, we find
\begin{align}\label{eqn:gjoegnqogqdasd}
     \E[\|\widehat{\bb}^t-\bb^t\|_1]=\tilde{O}\left(\sum_{k\notin \cI_\omega\cap \cI^t }\sqrt{\frac{ b_k^t}{n}}+ \frac{d\omega}{\sqrt{2^b n}}+{\frac{d}{\sqrt{2^b Tn}}}\right)
\end{align}
and 
\begin{equation*}
    \E[\|\widehat{\bb}^t-\bb^t\|_2^2]=\tilde{O}\left(\sum_{k\notin \cI_\omega\cap \cI^t }{\frac{ b_k^t}{n}}+ \frac{d\omega}{{2^b n}}+{\frac{d}{{2^b Tn}}}\right).
\end{equation*}
Note that $|(\cI_\omega\cap \cI^t)^c|\leq |\cI_\omega^c|+|(\cI^t)^c|\leq s/\omega+s=O(s/\omega)$ and 
\begin{align}
    \sum_{k\notin \cI_\omega\cap \cI^t }\sqrt{\frac{ b_k^t}{n}}&\leq \sqrt{|(\cI_\omega\cap \cI^t)^c|\sum_{k\notin \cI_\omega\cap \cI^t }{\frac{ b_k^t}{n}}}=\sqrt{\frac{|(\cI_\omega\cap \cI^t)^c|\max\{2^b,|(\cI_\omega\cap \cI^t)^c|\}}{2^bn}}\nonumber\\
    &=\sqrt{\frac{s/\omega\max\{2^b,s/\omega\}}{2^bn}}.\label{eqn:gvjoemfgqwftqw}
\end{align}
Plugging \eqref{eqn:gvjoemfgqwftqw} into \eqref{eqn:gjoegnqogqdasd} and using $\E[\|\widehat{\bp}^t-\bp^t\|_1]=O(\E[\|\widehat{\bb}^t-\bb^t\|_1])$, we find the conclusion in terms of the $\ell_1$ error. The results in terms of the $\ell_2$ error can be obtained similarly. 
\end{proof}

\section{Lower   Bounds}\label{app:lower-bounds}
In this section, we provide the proofs for the minimax lower bounds for estimating distributions under our heterogeneity model. We first re-state the
detailed version the lower bounds that apply to both the $\ell_2$ and $\ell_1$ errors.
\begin{theorem}[Detailed statement of Theorem \ref{thm:collab-lower-bound}]\label{thm:collab-lower-bound-combo}
For any---possibly interactive---estimation method, and for any $t\in[T]$ and $q=1,2$, we have
\begin{equation}\label{eqn:collab-lower-bound-combo}
    \inf_{\substack{(W^{t^\prime,[n]})_{t^\prime \in[T]}\\ \widehat{\bp}^{t}\, }} 
    \sup_{\substack{\bp^\star \in\cP_d\\  \{\bp^{t^\prime}:t^\prime\in[T]\}\subseteq\mathbb{B}_s(\bp^\star)}}
    \E[\|\widehat{\bp}^t-\bp^t\|_q^q]=\Omega\left(s^{1-q/2}\left(\frac{\max\{2^b,s\}}{2^bn}\right)^{q/2}+\frac{d}{(2^bTn)^{q/2}}\right).
\end{equation}
\end{theorem}

\begin{theorem}[Detailed statement of Theorem. \ref{thm:transfer-lower-bound}]\label{thm:transfer-lower-bound-combo}
For any---possibly interactive---estimation method, and a new cluster $\cC^{T+1}$, we have
\begin{align*}\label{eqn:transfer-lower-bound}
    &\inf_{\substack{(W^{t^\prime,[n]})_{t^\prime \in[T]}\\W^{T+1,[\tilde{n}]},\widehat{\bp}^{T+1} }}  \sup_{\substack{\bp^\star \in\cP_d\\ \{\bp^{t^\prime}:t^\prime\in[T+1]\}\subseteq\mathbb{B}_s(\bp^\star)}}\E[\|\widehat{\bp}^{T+1}-\bp^{T+1}\|_q^q]\\
    =&\Omega\left(s^{1-q/2}\left(\frac{\max\{2^b,s\}}{2^b\tilde{n}}\right)^{q/2}+\frac{d}{(2^bTn)^{q/2}}\right).
\end{align*}
\end{theorem}

We omit the proof of Theorem \ref{thm:transfer-lower-bound-combo} since  it follows from the same analysis as Theorem \ref{thm:collab-lower-bound-combo}.

\subsection{Proof of Theorem \ref{thm:collab-lower-bound-combo}}
As discussed in Section \ref{sec:lower-bounds}, we will prove \eqref{eqn:collab-lower-bound-combo} by considering two special cases of our sparse heterogeneity model: 
\begin{enumerate}
    \item The \emph{homogeneous} case where $\bp^1=\cdots=\bp^T=\bp^\star \in\cP_d$.
    \item The \emph{$s/2$-sparse} case where  $\|\bp^\star\|_0\leq s/2$ and $\|\bp^t\|_0\leq s/2$ for all $t\in[T]$.
\end{enumerate}
Therefore, it  naturally holds that 
\begin{equation}\label{eqn:gjojfqowrqw}
    \inf_{\substack{(W^{t^\prime,[n]})_{t^\prime \in[T]}\\ \widehat{\bp}^{t}\, }} 
    \sup_{\substack{\bp^\star \in\cP_d\\  \{\bp^{t^\prime}:t^\prime\in[T]\}\subseteq\mathbb{B}_s(\bp^\star)}}
    \E[\|\widehat{\bp}^t-\bp^t\|_q^q]\geq \inf_{\substack{(W^{t,[n]})_{t \in[T]}\\ \widehat{\bp}^{\star}\, }} 
    \sup_{\substack{\bp^\star \in\cP_d}}
    \E[\|\widehat{\bp}^\star-\bp^\star\|_q^q]
\end{equation}
and 
\begin{equation}\label{eqn:gjoejtomqwetfqw}
    \inf_{\substack{(W^{t^\prime,[n]})_{t^\prime \in[T]}\\ \widehat{\bp}^{t}\, }} 
    \sup_{\substack{\bp^\star \in\cP_d\\  \{\bp^{t^\prime}:t^\prime\in[T]\}\subseteq\mathbb{B}_s(\bp^\star)}}
    \E[\|\widehat{\bp}^t-\bp^t\|_q^q]\geq \inf_{\substack{(W^{t^\prime,[n]})_{t^\prime \in[T]}\\ \widehat{\bp}^{t}\, }} 
    \sup_{\substack{\|\bp^{t^\prime}\|_0\leq s/2\\\forall\,t^\prime\in[T]}}
    \E[\|\widehat{\bp}^t-\bp^t\|_q^q].
\end{equation}

For the first case, combining \eqref{eqn:gjojfqowrqw} with the existing lower bound result \cite[Cor 7]{Barnes2019LowerBF} and \cite[Thm 2]{Han2018DistributedSE} for the homogeneous setup, where all datapoints are generated by a single distribution, that for any estimation method (possibly based on interactive encoding),
\begin{equation*}
    \inf_{\substack{(W^{t,[n]})_{t \in[T]}\\ \widehat{\bp}^{\star}\, }} 
    \sup_{\substack{\bp^\star \in\cP_d}}
    \E[\|\widehat{\bp}^\star-\bp^\star\|_q^q]=\Omega\left(\frac{d}{(2^bTn)^{q/2}}\right),
\end{equation*}
we prove that the lower bound is at least of the order of the second term in \eqref{eqn:collab-lower-bound-combo}.

For the second case, without loss of generality, we assume $s$ is even. 
This can be achieved by considering $s-1$ instead of $s$, if necessary.
Recall that $\supp(\cdot)$ denotes the indices of non-zero entries of a vector.
Fixing any $t\in[T]$, we further consider the scenario where 
\begin{equation}\label{eqn:gjoemqowrq}
    \supp(\bp^t)\cap \left(\cup_{t^\prime\neq t}\supp(\bp^{t^\prime})\right)=\emptyset.
\end{equation} 
One example where  \eqref{eqn:gjoemqowrq} holds is when $\supp(\bp^t)\subseteq[s/2]$ and $\supp(\bp^{t^\prime})\subseteq\{s/2+1,\dots,d\}$ for all $t^\prime\neq t$. If \eqref{eqn:gjoemqowrq} holds, then the support of the datapoints generated by $\{\bp^{t^\prime}:t^\prime\neq t\}$ does not overlap with the support of those generated by $\bp^t$, and hence former are not informative for estimating $\bp^t$. 
Therefore, by further combining \eqref{eqn:gjoejtomqwetfqw} with the existing lower bound result \cite[Thm 2]{chen2021breaking} for the $s/2$-sparse homogeneous setup, where all datapoints are generated by a single $s/2$-sparse distribution,  that for any estimation method (possibly based on interactive encoding),
\begin{align*}
    \inf_{\substack{(W^{t,[n]})}} 
    \sup_{\substack{\|\bp^t\|_0\leq s/2\\\bp^t \in\cP_d}}
    \E[\|\widehat{\bp}^t-\bp^t\|_q^q]=&\Omega\left((s/2)^{1-q/2}\left(\frac{\max\{2^b,s/2\}}{2^bn}\right)^{q/2}\right)\\
    =&\Omega\left(s^{1-q/2}\left(\frac{\max\{2^b,s\}}{2^bn}\right)^{q/2}\right).
\end{align*}
Thus, we have
\begin{align*}
    &\inf_{\substack{(W^{t^\prime,[n]})_{t^\prime \in[T]}\\ \widehat{\bp}^{t}\, }} 
    \sup_{\substack{\bp^\star \in\cP_d\\  \{\bp^{t^\prime}:t^\prime\in[T]\}\subseteq\mathbb{B}_s(\bp^\star)}}
    \E[\|\widehat{\bp}^t-\bp^t\|_q^q]\geq \inf_{\substack{(W^{t^\prime,[n]})_{t^\prime \in[T]}\\ \widehat{\bp}^{t}\, }} 
    \sup_{\substack{\|\bp^{t^\prime}\|_0\leq s/2\\\forall\,t^\prime\in[T]}}
    \E[\|\widehat{\bp}^t-\bp^t\|_q^q]\\
    \geq &\inf_{\substack{(W^{t^\prime,[n]})_{t^\prime \in[T]}\\ \widehat{\bp}^{t}\, }} 
    \sup_{\substack{\|\bp^{t^\prime}\|_0\leq s/2,\,\forall\,t^\prime\in[T]\\\eqref{eqn:gjoemqowrq}\text{ holds} }}
    \E[\|\widehat{\bp}^t-\bp^t\|_q^q]= \inf_{\substack{(W^{t,[n]})}} 
    \sup_{\substack{\|\bp^t\|_0\leq s/2\\\bp^t \in\cP_d}}
    \E[\|\widehat{\bp}^t-\bp^t\|_q^q]\\
    =&\Omega\left(s^{1-q/2}\left(\frac{\max\{2^b,s\}}{2^bn}\right)^{q/2}\right).
\end{align*}
This proves that the lower bound is at least of the order of the first term in \eqref{eqn:collab-lower-bound-combo}.
Overall, we conclude the desired result.

\section{Supplementary Experiments}\label{app:experiments}

\paragraph{Truncated geometric distribution. }
We consider the truncated geometric distribution with parameter $\beta \in (0, 1)$, $\mathbf{p}^\star = \frac{1 - \beta}{1 - \beta^d}(1, \beta, \dots, \beta^{d - 1})$, as the central distribution and repeat the experiment in Section \ref{sec:synthetic}.
We use $d = 300, \beta = 0.95, b = 2$ and vary $n, T, s$.
Figure \ref{fig:syntheticgeom} summarizes the results.
As in Section \ref{sec:synthetic}, we observe that our methods outperform the baseline methods in most cases, especially when $s$ is small.
Also, we see the benefit of collaboration, \ie, decreasing trend of the error as $T$ increases, only in our methods.

\begin{figure}
    \centering
    \includegraphics{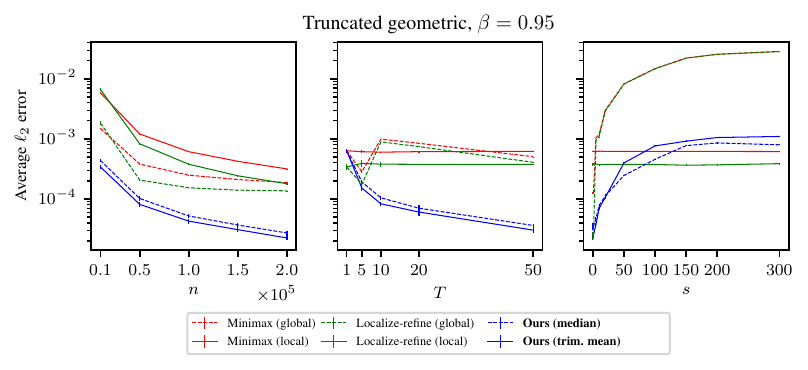}
   \caption{Average $\ell_2$ estimation error in synthetic experiment using the truncated geometric distribution. (Left):  Fixing $s = 5, \,T = 30$ and varying $n$.
  (Middle): Fixing $s = 5,\,n = 100,000$ and varying $T$.
  (Right): Fixing  $T = 30, \,n = 100,000$ and varying $s$.
  The standard error bars are obtained from 10 independent runs.}
    \label{fig:syntheticgeom}
\end{figure}

\paragraph{Hyperparmeter selection.}
We provide additional experiments using different hyperparameters $\alpha$ and $\omega$ from discussed in Section \ref{sec:synthetic}.
All other settings are identical to Section \ref{sec:synthetic}.
We test the hyperparameters $(\alpha, \omega) = (2^r \ln(n), 0.1)$  and $(\alpha, \omega) = (\ln(n), \omega)$, where $r \in \{-5, -4, \dots, 4\}$ and $\omega \in \{0.05, 0.1, \dots, 0.25\}$, respectively.
Figure \ref{fig:additionalexperiments} summarizes the results.

We find that setting the threshold $\alpha$  too small leads to replacing almost all coordinates of the central estimate $\widehat{\mathbf{p}}^\star$ with local ones.
In the extreme case of $\alpha \approx 0$, our method is essentially returns the local minimax estimates.
On the other hand, we observe that the performance of our method is less sensitive to the trimming proportion $\omega$.

While the choice of $\alpha$ is crucial to the performance of our method, we argue that it is possible to select a reasonably good $\alpha$ by checking the number of fine-tuned entries, \ie,
$$ \frac{1}{T} \sum_{t = 1}^T \left\vert \left\{ k \in [d]:  |[\widehat{\bb}^\star]_k-[\widehat{\bb}^t]_k| > \sqrt{\frac{\alpha[\widehat{\bb}^t]_k}{n}} \right\} \right\vert.$$
In Figure \ref{fig:alphaheuristic}, we observe that more than half ($d / 2 = 150$) of the entries are fine-tuned when $r \in\{ -5, -4, -3\}$.
These correspond to the three curves in the top left of Figure \ref{fig:additionalexperiments} that perform no better than the baseline methods.
In conclusion, by selecting $\alpha$ such that the number of fine-tuned entries are small enough compared to $d$, it is possible to reproduce the results in Section \ref{sec:experi}.

\begin{figure}
    \centering
    \includegraphics[scale=0.95]{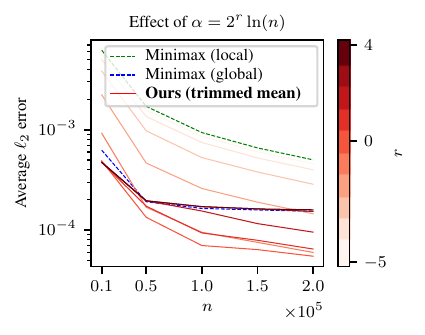}
    \includegraphics[scale=0.95]{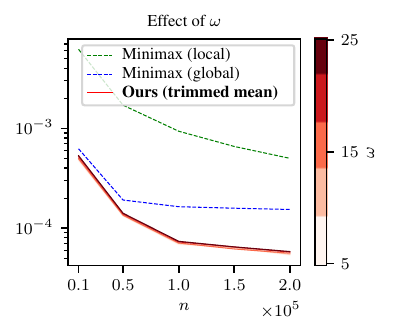}
    \includegraphics[scale=0.95]{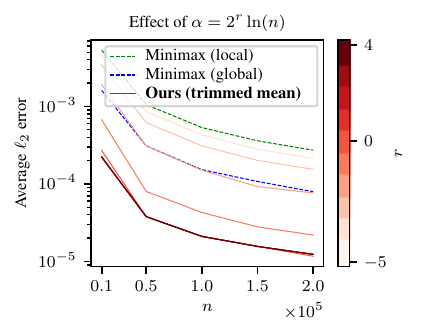}
    \includegraphics[scale=0.95]{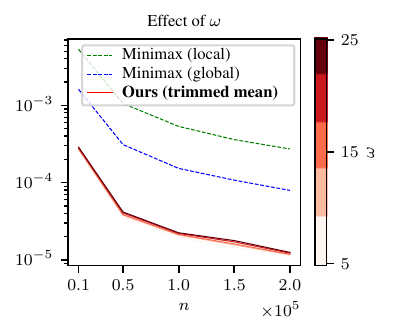}
    \caption{Effect of the hyperparameters $\alpha$ and $\omega$. The top row shows results for the uniform distribution and the bottom row shows the results for the truncated geometric distribution with $\beta = 0.8$.}
    \label{fig:additionalexperiments}
\end{figure}

\begin{figure}
    \centering
    \includegraphics{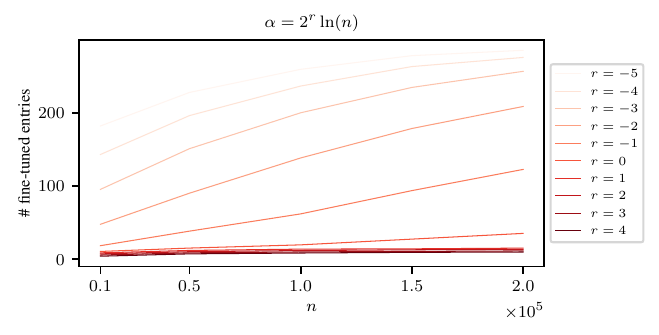}
    \caption{Average number of fine-tuned entries for different values of $\alpha = 2^r \ln (n)$. We use the trimmed mean with $\omega = 0.1$ and the uniform distribution with $d = 300$. This corresponds to the top left of Figure \ref{fig:additionalexperiments}.}
    \label{fig:alphaheuristic}
\end{figure}

\end{document}